\documentclass[journal]{IEEEtran}
\usepackage{xcolor,colortbl,bbm}
\usepackage{algorithmic,algorithm}
\usepackage{graphicx, subcaption}
\usepackage{balance}
\usepackage{amssymb}
\usepackage{enumitem}

\usepackage{cite}

\ifCLASSINFOpdf
\else
\fi
\usepackage{amssymb, amsmath, amsthm,bm,dsfont, float}
\DeclareFontFamily{OT1}{pzc}{}
\DeclareFontShape{OT1}{pzc}{m}{it}{<-> s * [1.10] pzcmi7t}{}
\DeclareMathAlphabet{\mathpzc}{OT1}{pzc}{m}{it}  

\newcommand{\diag}[1]{\text{diag}\big\{#1\big\}}
\newcommand{\col}[1]{\text{col}\left\{#1\right\}}
\newcommand{\row}[1]{\text{row}\left\{#1\right\}}
\newcommand{\grad}[1]{\nabla_{#1^{\tran}} }
\newcommand{\ws}{\bm{{\scriptstyle\mathcal{W}}}}

\newcommand{\wse}{\widetilde{\ws}}
\newcommand{\w}{\bm{w}}
\newcommand{\x}{\bm{x}}

\newcommand{\we}{\widetilde{\w}}
\newcommand{\eqdef}{\:\overset{\Delta}{=}\:}
\DeclareMathOperator*{\argmin}{argmin}

\newcommand{\tran}{{\sf T}}
\newcommand{\A}{\bm{\mathcal{A}}_{iT+t}}
\newcommand{\M}{\bm{\mathcal{M}}_{i}}
\newcommand{\Hi}{\bm{\mathcal{H}}_{iT+t-1}}

\newcommand{\F}{\mathcal{F}_{iT+t-1}}
\newcommand{\Abar}{\bar{\mathcal{A}}}

\newcommand{\swb}{{\boldsymbol{\scriptstyle{\mathcal{W}}}}}
\newcommand{\sxb}{{\boldsymbol{\scriptstyle{\mathcal{X}}}}}

\newtheorem{theorem}{Theorem}
\newtheorem{assumption}{Assumption}
\newtheorem{lemma}{Lemma}

\newtheorem{remark}{Remark}

\allowdisplaybreaks

\definecolor{Gray}{gray}{0.8}
\definecolor{LightCyan}{rgb}{0.88,1,1}

\makeatletter
\def\endthebibliography{%
	\def\@noitemerr{\@latex@warning{Empty `thebibliography' environment}}%
	\endlist
}
\makeatother

\begin{document}
%
\title{Diffusion Learning with Partial Agent Participation and Local Updates}
%
%
%

\author{Elsa~Rizk\IEEEauthorrefmark{1},~\IEEEmembership{}
        Kun~Yuan\IEEEauthorrefmark{2},~\IEEEmembership{}
        and~Ali~H. Sayed\IEEEauthorrefmark{1}~\IEEEmembership{}
\thanks{\IEEEauthorrefmark{1}The authors are with the School of Engineering, École Polytechnique Fédérale de Lausanne, Lausanne, Switzerland (e-mail: elsa.rizk@alumni.epfl.ch; ali.sayed@epfl.ch).
	\IEEEauthorrefmark{2}The author is with the Center for Machine Learning Research, Peking University, Beijing, China (e-mail: kunyuan@pku.edu.cn).
 A short conference version appears in \cite{rizk2024asynchronous}. \textit{(Corresponding author: Ali H. Sayed.)}
}
}

\maketitle

\begin{abstract}
Diffusion learning is a framework that endows edge devices with advanced intelligence. By processing and analyzing data locally and allowing each agent to communicate with its immediate neighbors, diffusion effectively protects the privacy of edge devices, enables real-time response, and reduces reliance on central servers. However, traditional diffusion learning relies on communication at every iteration, leading to communication overhead, especially with large learning models. 
Furthermore, the inherent volatility of edge devices, stemming from power outages or signal loss, poses challenges to reliable communication between neighboring agents. To mitigate these issues, this paper investigates an enhanced diffusion learning approach incorporating local updates and partial agent participation. Local updates will curtail communication frequency, while partial agent participation will allow for the inclusion of agents based on their availability. We prove that the resulting algorithm is stable in the mean-square error sense and provide a tight analysis of its  Mean-Square-Deviation (MSD) performance. Various numerical experiments are conducted to illustrate our theoretical findings.
\end{abstract}

\begin{IEEEkeywords}
distributed learning, federated learning, diffusion algorithm, local updates, partial agent participation
\end{IEEEkeywords}

%
\IEEEpeerreviewmaketitle

\section{Introduction}

\IEEEPARstart{E}{dge} intelligence refers to an emerging paradigm that brings data processing and analysis capabilities near data generation sources, rather than transmitting data to centralized cloud servers for remote processing. This approach enables faster response, enhanced data privacy, reduced bandwidth usage, and improved reliability, as the data is processed locally on edge devices such as IoT devices, smartphones, wearables, and edge servers. Edge intelligence finds applications in various domains, such as smart cities~\cite{khan2020edge}, industrial automation~\cite{dai2019industrial}, autonomous vehicles~\cite{zhang2019mobile}, healthcare~\cite{hayyolalam2021edge}, and surveillance systems~\cite{rajavel2022iot}, where real-time decision-making or immediate insights and actions are crucial.

There are several challenges to fully utilizing local data collected by edge agents to enhance their intelligence. First, the data collected by individual edge agents is often limited in size, making it difficult for a single agent to train an effective machine learning model based on its local data alone. Second, this local data is usually sensitive and private, and edge agents are typically reluctant to transmit it to a central server for large-scale model training. Decentralized learning solves these challenges, enabling all agents to collaborate on training a global model without sharing their private local data with a remote server.

In decentralized learning, edge agents are connected via a network topology and communicate only with their immediate neighbors. Two popular and fundamental schemes in decentralized learning are: consensus \cite{degroot1974reaching,nedic2009distributed,boyd2006randomized,kar2008distributed} and diffusion \cite{chen2012diffusion,tu2012diffusion,chen2015learning,sayed2014adaptive,sayed2014adaptation}. One key distinction between the two classes of strategies is that the consensus updates are asymmetrical. In this approach, the starting point for the gradient-descent step differs from the location where the gradient vector is evaluated. Previous research (see, e.g., \cite{tu2012diffusion,sayed2014adaptation,sayed2014adaptive}) has demonstrated that this asymmetry reduces the stability range of consensus implementations compared to diffusion solutions, particularly in scenarios requiring continuous learning and adaptation. Recent efforts have extended consensus and diffusion to scenarios with data heterogeneity \cite{shi2015extra,yuan2018exact,yuan2018exact2,xu2015augmented,nedic2017achieving,di2016next,tang2018d}, non-smooth regularizers \cite{shi2015proximal,alghunaim2019linearly,xu2021distributed}, or directed network topologies \cite{nedic2014distributed,xin2018linear,pu2020push}.

While the aforementioned algorithms effectively enable agents to collaboratively train machine learning models, they necessitate each agent to transmit high-dimensional optimization variables to neighbors at every iteration. Given the substantial size of these high-dimensional variables, such frequent communication incurs significant overhead, thereby hindering algorithmic efficiency and scalability. To mitigate this issue, several recent works introduce the local update technique, which is widely employed in federated learning algorithms such as FedAvg~\cite{konevcny2016federated} and Scaffold~\cite{karimireddy2020scaffold}, to fully decentralized algorithms. By allowing agents to perform multiple local updates before exchanging information with neighbors, the local update strategy significantly reduces communication frequency and consequently alleviates the associated overhead in decentralized learning frameworks \cite{liu2024decentralized,nguyen2023performance,ge2023gradient,alghunaim2024local}. 

Despite the effectiveness of local updates in reducing communication costs, existing works \cite{liu2024decentralized, nguyen2023performance, ge2023gradient, alghunaim2024local} necessitate that all agents be available for each iteration and participate in the entire learning procedure. This prerequisite is restrictive, as device agents are typically fragile and can become unavailable due to power outages or signal loss. Consequently, given the inevitability of intermittent unavailability, one cannot expect them to participate in every iteration. The issue of partial agent participation has been studied in the context of federated learning \cite{karimireddy2020scaffold, chen2022optimal, rizk2022federated}. However, it remains unclear how to address this issue in fully decentralized settings. 

To address this question, this paper proposes the first decentralized learning algorithm that enables both local updates and partial agent participation. During each iteration, a subset of agents participates in the training process. Upon activation, these agents perform a sequence of local stochastic gradient updates and then communicate with their active neighbors. By iteratively repeating this procedure, the proposed algorithm achieves both communication efficiency and robustness to partial agent participation, providing a more efficient and resilient solution for decentralized learning with edge agents prone to intermittent unavailability. Additionally, the proposed algorithm is highly versatile, encompassing FedAvg with partial agent participation by specifying certain network topologies and activation strategies. In summary, this paper makes the following contributions: 
\begin{itemize}[leftmargin = 2em]
    \item We propose the first decentralized learning algorithm that supports both local updates and partial agent participation. This algorithm reduces communication overhead and enhances robustness to agent unavailability.

    \vspace{1mm}
    \item We justify our algorithm's versatility by showing that it encompasses various existing algorithms through specific network topologies and activation strategies.

    \vspace{1mm}
    \item We demonstrate that the resulting algorithm is stable in the mean-square error sense. We derive a  closed-form mean-square deviation (MSD) expression, which is directly applicable to federated learning algorithms. Prior to our work, the existing literature could only establish the order of the steady-state performance for these algorithms, rather than provide precise expressions.

    \vspace{1mm}
    \item We conduct extensive experiments to illustrate the theoretical findings. 
\end{itemize}

\noindent \textbf{Notation.} We use $\mathrm{col}\{x_1,\cdots, x_K\}$ to denote a column vector by stacking $x_1,\cdots, x_K$, and $\mathrm{row}\{x_1,\cdots, x_K\}$ to denote a row vector. We use $\mathds{1}_K \eqdef \mathrm{col}\{1,\cdots, 1\}\in \mathbb{R}^K$  and let 
$\nabla_{w^\tran} J(w)\eqdef \mathrm{col}\{\partial J(w)/\partial w_1, \cdots, \partial J(w)/\partial w_K\}\in \mathbb{R}^K$. We write $I_K \in \mathbb{R}^{K\times K}$ to denote the identity matrix. The Kronecker product is denoted by $\otimes$.



\section{Problem Setup and Diffusion Learning}
We consider a network of \( K \) agents collaborating to solve the following problem in a fully distributed manner:
\begin{align}\label{eq:optProb}
  \min_{w\in\mathbb{R}^M} \frac{1}{K}\sum_{k=1}^K\left\{J_k(w)  \eqdef \frac{1}{N_k}\sum_{n=1}^{N_k} Q_k(w;x_{k,n})\right\},
\end{align}
where each agent \( k \) maintains a local risk \( J_k(w) \), defined as the empirical average of the local loss functions \( Q_k(w; x_{k,n}) \) over its local dataset \(\{ x_{k,n} \}_{n=1}^{N_k} \). The quantity \( N_k \) indicates the size of the local dataset. The underlying network is associated with a combination matrix \( A = [a_{\ell k}] \in \mathbb{R}^{K \times K} \), where the weight \( a_{\ell k} \) scales the information sent from agent \(\ell\) to agent \( k \). This paper makes the following standard assumptions.

\begin{assumption}[\sc{Combination matrix}]\label{assum:combMat}
    The combination matrix $A$ is symmetric and left-stochastic, namely:
    \begin{align} 
    a_{\ell k} = a_{k\ell} \geq 0, \quad \quad  A^\tran \mathds{1}_K = \mathds{1}_K.
\end{align}
Additionally, we assume that $A$ is primitive; that is, there exists an integer $m$ such that all entries of $A^m$ are strictly positive.
\qed
\end{assumption}
A sufficient condition for $A$ to be primitive is that $A$ is strongly connected, namely, there exists a path of nonzero weights linking any two agents $k$ and $\ell$ (the path in both directions need not be identical) and, moreover, there exists at least one agent $k^o$ with a nontrivial self-loop, i.e., $a_{k^o, k^o} >0$.
Additionally, Assumption \ref{assum:combMat} implies that $A$ is doubly stochastic, which in turn implies that its Perron eigenvector is $p = (1/K)\mathds{1}_K$.

\begin{assumption}[\sc{Risk and loss functions}]\label{assum:fct}
    The empirical risks $J_k(w)$
are $\nu-$strongly convex for some $\nu > 0$, and the loss functions $Q_k(w; x_{k,n})$ are convex and twice differentiable over $w$, namely, it holds for any $w_1$ and $w_2$ that:
    \begin{align}
        &J_k (w_2)\hspace{-0.7mm}\geq \hspace{-0.7mm} J_k(w_1) \hspace{-0.7mm}+\hspace{-0.7mm} \grad{w}J_k(w_1)(w_2\hspace{-0.7mm}-\hspace{-0.7mm}w_1) \hspace{-0.7mm}+\hspace{-0.7mm} \frac{\nu}{2} \Vert w_2\hspace{-0.7mm}-\hspace{-0.7mm}w_1\Vert^2\hspace{-0.7mm}, \\
       & Q_k(w_2;x_{k,n})\hspace{-1mm}  \geq \hspace{-1mm} Q_k(\hspace{-0.4mm}w_1;x_{k,n}\hspace{-0.4mm}) 
        \hspace{-0.8mm}+\hspace{-0.8mm} \grad{w}Q_k(\hspace{-0.4mm}w_1\hspace{-0.4mm};\hspace{-0.4mm}x_{k,n}\hspace{-0.2mm})\hspace{-0.2mm}(\hspace{-0.4mm}w_2\hspace{-1mm}-\hspace{-1mm}w_1\hspace{-0.4mm}).
    \end{align}
    We further assume that the loss functions have $\delta-$Lipschitz continuous gradients for any $x_{k,n}$:
    \begin{align}
        \Vert \nabla_{w^\tran}Q_k(w_2;x_{k,n}) \hspace{-0.3mm}-\hspace{-0.3mm} \nabla_{w^\tran}Q_k(w_1;x_{k,n})\Vert \hspace{-0.3mm}\leq\hspace{-0.3mm} \delta \Vert w_2 \hspace{-0.5mm}-\hspace{-0.5mm} w_1\Vert .
    \end{align}
    \qed
\end{assumption}
\noindent Assumption \ref{assum:fct} implies that the Hessian matrix is bounded from below and above:
\begin{align}
    \nu I \leq \nabla_{w}^2J_{k}(w) \leq \delta I.
\end{align}

\begin{assumption}[\sc{Smooth Hessians}]\label{assum:hess}
    The Hessians are locally Lipschitz in a small neighborhood around the optimal model $w^o$ (which is defined in \eqref{eq:optMod} of Section~\ref{sec:drift}), namely, there exists $\kappa > 0$ for small $\Delta w$:
    \begin{align}
        \Vert \nabla_{w}^2J_k(w^o + \Delta w) -\nabla_{w}^2J_k(w^o) \Vert \leq \kappa\Vert \Delta w \Vert . 
    \end{align}
    \qed
\end{assumption}
To solve problem \eqref{eq:optProb}, we employ the well-known diffusion learning algorithm \cite{chen2012diffusion,tu2012diffusion,chen2015learning,sayed2014adaptive,sayed2014adaptation}, where each agent \( k \) executes the following recursions in a distributed manner:
\begin{align}
\bm{\psi}_{k,i} &= \w_{k,i-1} - \mu \nabla_{w^\tran} Q(\w_{k,i-1}; \x_{k,i}), \label{diff-1}\\
\w_{k,i} &= \sum_{\ell \in \mathcal{N}_k} a_{\ell k} \bm{\psi}_{\ell,i}. \label{diff-2}
\end{align}
Here, the index \( i \) denotes the iteration, \( \mathcal{N}_k \) represents the set of neighbors of agent \( k \) (including \( k \) itself), and the random variable \( \x_{k,i} \) is sampled uniformly from the dataset \( \{x_{k,1}, \ldots, x_{k,N_k}\} \). Throughout this paper, we utilize a constant step-size $\mu$ to facilitate continuous adaptation and learning in response to potential drifts of the global minimizer, which may occur due to changes in the statistical properties of the data. If we  collect the quantities from across the network into block vectors, namely, 
\begin{align}
    \swb_i &\eqdef \col{\w_{1,i}, \cdots, \w_{K,i}} \in \mathbb{R}^{KM} \\
    \sxb_i &\eqdef \col{\x_{1,i}, \cdots, \x_{K,i}} \in \mathbb{R}^{KM} \\
    \nabla \mathcal{Q}(\swb_{i-1};\sxb_i)&\eqdef \text{col}\{\nabla_{w^\tran} Q_1(\w_{1,i-1}; \x_{1,i}),\cdots, \nonumber \\
    &\quad \quad  \nabla_{w^\tran} Q_K(\w_{K,i-1}; \x_{K,i})\} \hspace{-0.5mm} \in \mathbb{R}^{KM} \\
    \nabla \mathcal{J}(\swb_{i-1})&\eqdef \text{col}\{\nabla_{w^\tran} J_1(\w_{1,i-1}),\cdots, \nonumber \\
    &\quad \quad \quad \ \nabla_{w^\tran} J_K(\w_{K,i-1})\} \hspace{-0.5mm} \in \mathbb{R}^{KM} \\
    \mathcal{A} &\eqdef A \otimes I_M,
\end{align}
then the diffusion learning algorithm \eqref{diff-1}--\eqref{diff-2} can be written in a more compact manner as:
\begin{align}\label{diff-compact}
\swb_i = \mathcal{A}\Big(\swb_{i-1} - \mu \nabla \mathcal{Q}(\swb_{i-1};\sxb_i)\Big).
\end{align}

\section{Algorithm development}

\subsection{Diffusion learning with local updates}

Diffusion learning allows agents to collaboratively train machine learning models efficiently; however, when $M$ is large, the implementation would require agents to send high-dimensional variables to their neighbors at {\em every} iteration --- see recursion \eqref{diff-2}. The large dimension of the variables leads to communication overhead, which impacts the efficiency and scalability of the algorithms. To mitigate this issue, we propose incorporating local updates into the diffusion process. This modification allows communication to occur only once every \( T \) local updates, thereby significantly reducing communication frequency. More specifically, within each $i$-th block of indices, the agents perform independent local updates for $t=1,2,\ldots, T-1$ times. At $t=T$, the agents perform one final local update and subsequently combine their neighborhood weights. The next $(i+1)$-th block of indices is then launched and the process repeats. This iterative process alternates between a series of independent local updates and a combination step every \( T \) steps.

To implement this approach, we introduce two sets of indices to the algorithm: index $i$ represents the block iteration, while index $t$ denotes the local update step within the block. To enable local updates, we allow the combination matrix to vary with iterations as follows:
\begin{align}\label{A-LU}
\bm{A}_{iT+t} = 
\begin{cases}
I_K, & \mbox{if $t\neq T$,} \\
A, & \mbox{if $t = T$.}
\end{cases}
\end{align}
That is, the combination matrix is set to the identity matrix during the block, and to $A$ at the end of the block. Using this time-varying combination matrix, diffusion learning with local updates can be written as follows: 
\begin{align}\label{eq:diff-compact-local-update}
\swb_{iT+t} &= \A \Big(\swb_{iT+t-1} - \mu \nabla \mathcal{Q}(\swb_{iT+t-1};\sxb_{{iT+t}})\Big), \nonumber \\
&\quad \forall t = 1,2,\cdots, T, \quad \forall i=0,1,2,\cdots.
\end{align}
Here, the subscript $iT+t$ for $t=1,2,\cdots, T$ represents the $t$-th local update during the $i$-th block iteration of the algorithm, and matrix $\A \eqdef \bm{A}_{iT+t} \otimes I_M$. Within every block iteration $i$, all agents separately run local updates during the first $T-1$ time instances. During the last time instance $T$, all agents run one final update followed by a combination step across neighbors --- see Figure~\ref{fig:time-scale} for an illustration of the two time-scale update process. 

\begin{figure}[h!]
    \centering
    \includegraphics[width=0.47\textwidth]{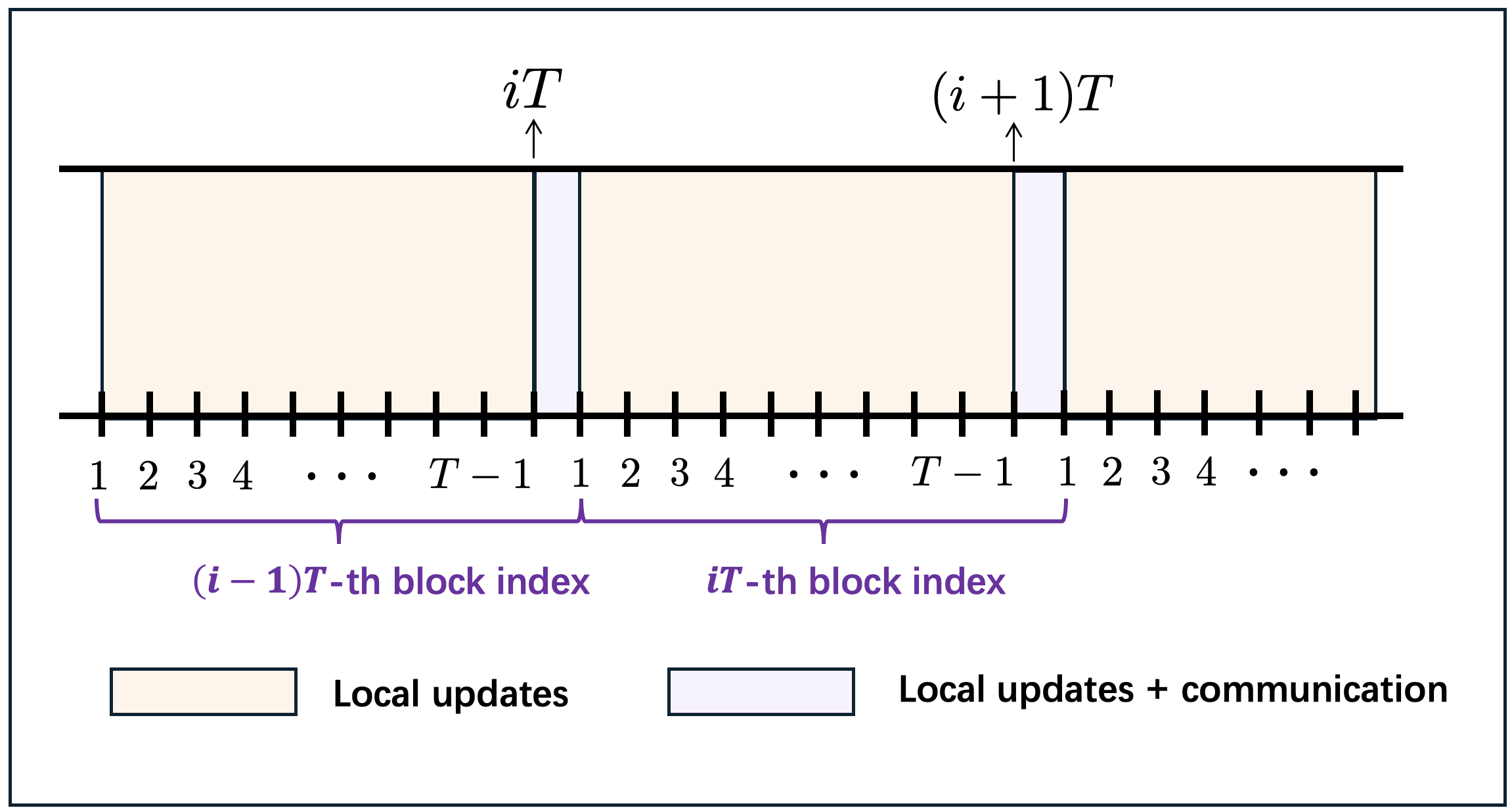}
    \caption{\small Illustration of the time scales for diffusion learning with local updates, i.e., recursion \eqref{eq:diff-compact-local-update}. 
    }
    \label{fig:time-scale}
\end{figure}

\subsection{Enabling partial agent participation}

Despite the efficacy of local updates in reducing communication costs, recursion \eqref{eq:diff-compact-local-update} requires all agents to be available for each iteration and participate throughout the learning process. This requirement is restrictive, as agents are vulnerable to disruptions like power outages or signal loss, making them intermittently unavailable. Consequently, it is unrealistic to expect continuous participation from all agents in every iteration. To address this limitation, we introduce a partial agent participation strategy to recursion \eqref{eq:diff-compact-local-update}.

Again, we model partial agent participation using a time-varying combination matrix. We assume that only a subset of agents is active during the \(i\)-th block iteration. An inactive agent is unable to transmit information to its neighbors. This scenario effectively results in sampling a subset of agents and edges from the original network topology.

\begin{figure*}[t!]
    \centering
\includegraphics[width=1\textwidth]{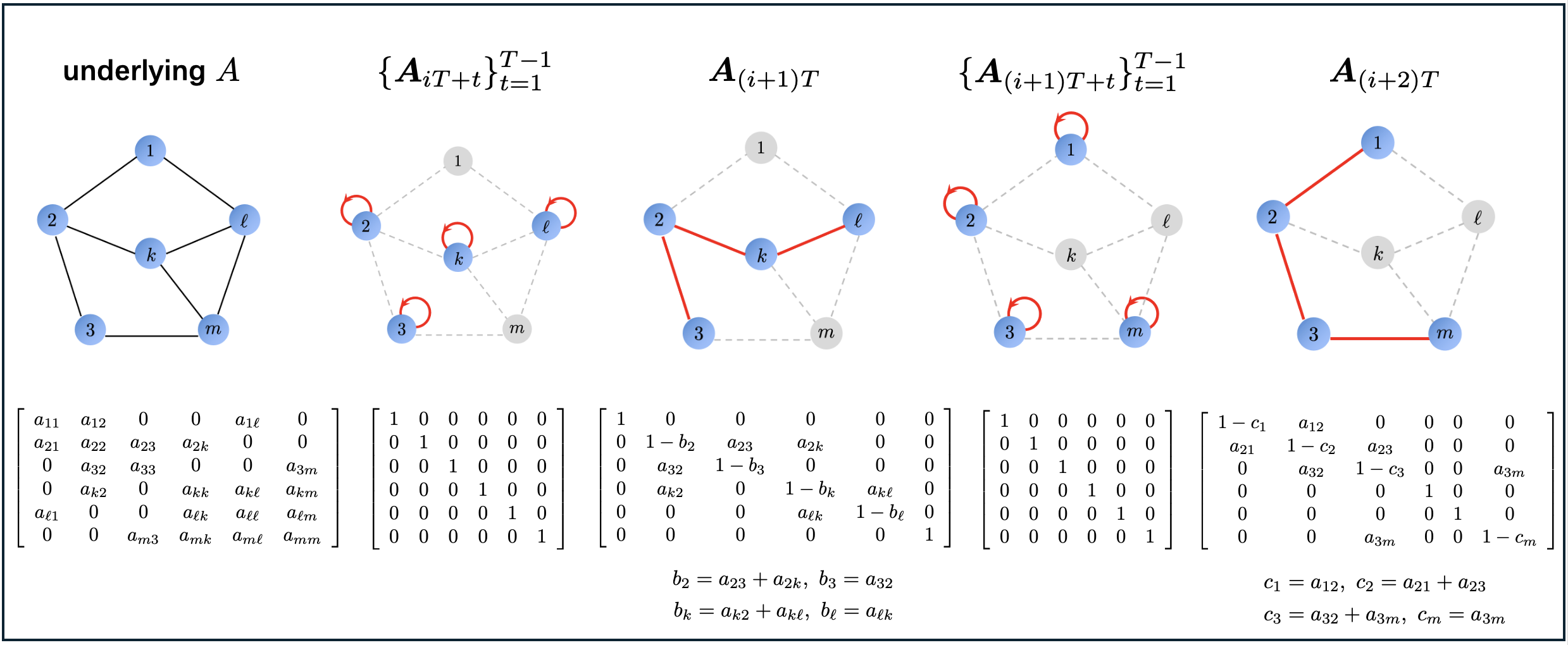}
    \caption{\small {Evolution of the network topology during the $i$-th and $(i+1)$-th block index. The first figure shows the underlying graph. The second and fourth figures depict the topology during the local update steps, with all links turned off and inactive agents dimmed. The third and fifth figures illustrate the topology during the collaboration stage, where only the links between active agents are maintained.}}
    \label{fig:netEvol}
\end{figure*}

\vspace{1mm}
\noindent \textbf{Agent activation.} 
At the start of the $i$-th block iteration, we assume agent $k$ participates with probability $q_k$, where $q_k \eqdef \mathbb{P}(k \text{ is active})$. For an inactive agent, its current model will be set equal to its previous model, resulting in a step-size of 0. In contrast, the step-size of an active agent will be set to a constant regardless of $t$. Therefore, we define the random step-size of agent $k$ as:
\begin{align}\label{mu-agent-activation}
    \bm{\mu}_{k,iT+t} = \begin{cases}
        \mu ,& k \text{ active at $i$-th {block} iteration}, \\
        0, & k \text{ inactive at $i$-th {block} iteration $i$},
    \end{cases}
\end{align}
where $t=1,\cdots, T$. Note that the definition does not depend on the local updates, we can drop the notation $t$  and simply write $\bm{\mu}_{k,iT+t}$ as $\bm{\mu}_{k,i}$.

\vspace{1mm}
\noindent \textbf{Time-varying topology.} 
During the first $T-1$ steps at iteration $i$, we let each active agent conduct local updates by setting: 
\begin{align}\label{28zba}
\bm{A}_{iT+t} = I_K \quad \mbox{when} \quad t\neq T.
\end{align}
At $t=T$, the active agent performs a final local update step followed by a combination step, assigning non-zero weights $\bm{a}_{\ell k, iT} = a_{\ell k}$ to its participating neighbors and adjusting its self-weight to maintain a stochastic combination matrix: 
\begin{align}\label{z71}
    \bm{a}_{\ell k, (i+1)T} \hspace{-0.5mm}=\hspace{-0.5mm} \begin{cases}
        a_{\ell k}, &\hspace{-1mm} {\mbox{$k$ active, $\ell\in \mathcal{N}_k(i)/\{k\}$}}
        \\
        1 \hspace{-0.5mm}-\hspace{-0.5mm} {\sum\limits_{j \in \mathcal{N}_k(i)/\{k\}} {a}_{jk}} ,&\hspace{-1mm} {\mbox{$k$ active, $\ell = k$}}
        \\
        1, &\hspace{-1mm} k \text{ inactive}, \ \ell = k \\
        0, &\hspace{-1mm} \text{otherwise}
    \end{cases}
\end{align}
{where $\mathcal{N}_k(i)\eqdef\{j|j \hspace{-0.5mm}\in\hspace{-0.5mm} \mathcal{N}_k$ and $j$ is active at the $i$-th block iteration$\}$ denotes the set of active neighbors of agent $k$.} We can verify that the matrix $\bm{A}_{iT+t}$ remains {\em doubly stochastic} for every $t=1,2,\cdots T$ and $i$ due to the symmetric property of the underlying matrix. This fact will be important when showing the diffusion algorithm is stable in the mean-square error sense. Figure~\ref{fig:netEvol} illustrates the time-varying topology and combination matrix associated with diffusion learning with both local updates and partial agent participation.

The following lemma establishes the expected value of the combination matrix.

\begin{lemma}[\sc {Network topology in expectation}]
\label{lemm:topo}
    In expectation, the network has the following {doubly-stochastic} combination matrix:
    \begin{align}
        \bar{A}_{iT+t} &\eqdef \mathbb{E}[\bm{A}_{iT+t} ] = [\bar{a}_{\ell k,{iT+t}}], \\
        \bar{a}_{\ell k,{iT+t}} &\eqdef \begin{cases}
             q_{\ell}q_k a_{\ell k} & \ell\neq k,\: t=T \\
             1 - \sum\limits_{\ell \neq k} q_{\ell}q_k a_{\ell k} & \ell = k,\: t=T \\
             0 & \ell \neq k,\: t\neq T \\
             1 &  \ell = k,\: t\neq T
        \end{cases}
    \end{align}
    {Furthermore, it holds that 
    \begin{align}
        \bar{M} &\eqdef \mathbb{E}[\bm{M}_i] = \mathrm{diag}(\mu q_1,\cdots, \mu q_K), \\
        \mathbb{E}[\bm{A}_{iT+t} \bm{M}_i] &= \mu(\bar{A}_{iT+t} - I) + \bar{M}, \label{q23zand210}
    \end{align}
    where $\bm{M}_i = \mathrm{diag}(\bm{\mu}_{1,i},\cdots, \bm{\mu}_{K,i}) \in \mathbb{R}^{K\times K}$. 
    }
\end{lemma}
\begin{proof}
    {See Appendix \ref{app:lemmtopo} for proof details.} 
\end{proof}

\vspace{1mm}
\noindent \textbf{Algorithm development.} With the agent activation depicted in \eqref{mu-agent-activation} and the time-varying topology illustrated in \eqref{28zba} and \eqref{z71}, diffusion learning with partial agent participation and local updates can be written as follows: 
\begin{align}\label{diff-compact-local-update-partial}
\swb_{iT+t} &= \A \Big(\swb_{iT+t-1} - \M \nabla \mathcal{Q}(\swb_{iT+t-1};\sxb_{{iT+t}})\Big), \nonumber \\
&\quad \forall t = 1,2,\cdots, T, \quad \forall i=0,1,2,\cdots.
\end{align}
where
\begin{align} \label{z72bvba}
    \A \eqdef \bm{A}_{iT+t} \otimes I_M, \quad \M \eqdef \bm{M}_i \otimes I_M.
\end{align}
According to \eqref{mu-agent-activation}--\eqref{z72bvba}, the active agents will perform \(T\) local updates followed by a combination step with their active neighbors. In contrast, the inactive agents remain unchanged during the entire \(i\)-th block iteration. The implementation details of \eqref{diff-compact-local-update-partial} are listed in Algorithm~\ref{alg:assynATC}.

\begin{algorithm}[t!]
\begin{algorithmic}
\caption{Diffusion Learning with Local Updates and Partial Agent Participation}\label{alg:assynATC}
\STATE{
\textbf{initialize} $w_{k,0} = 0$ for any $k=1,\cdots, K$\;}
\FOR{each block iteration $i={0},1,2,\cdots$}\STATE{
Every agent $k$ is active with probability $q_k$ \\
Every active agent adjust weights according to \eqref{z71}
{
\FOR{each agent $k = 1,\cdots, K$} \STATE {
\vspace{-0.4cm}
\FOR{each local step $t=1,2,\cdots,T$}
\STATE{
\vspace{-0.4cm}
\IF{agent $k$ is active}
\STATE \hspace{-3mm} Sample $n \in \{1,2,\cdots,N_k\}$  
\STATE \hspace{-3mm}  $\w_{k,iT+t} = \w_{k,iT+t-1} \hspace{-0.5mm} - \hspace{-0.5mm} \mu\grad{w}Q_k(\hspace{-0.5mm}\w_{k,iT+t-1};\bm{x}_{k,n})$
\\
\ELSE
    \STATE \hspace{-3mm} $\w_{k,iT+t} = \w_{k,iT+t-1}$
\ENDIF 
}\ENDFOR
}\ENDFOR}
\FOR{each active agent $k$}
    \STATE {$\bm{\psi}_{k,i+1} = \sum\limits_{\ell=1}^K \bm{a}_{\ell k, (i+1)T} \w_{\ell, (i+1)T}$}
    \STATE {Set $\w_{k,(i+1)T} \leftarrow \bm{\psi}_{k,i+1}$}
\ENDFOR
}\ENDFOR
\end{algorithmic}
\end{algorithm}




\subsection{Drifts in optimal solution}\label{sec:drift}
In standard diffusion learning \eqref{diff-compact}, each agent solves the desired problem \eqref{eq:optProb}. However, this is not the case for diffusion learning with partial client participation. Since each agent \( k \) is now activated with probability \( q_k \), the variables \( \{\bm{w}_{k,iT}\}_{k,i} \) generated by Algorithm~\ref{alg:assynATC} will converge in expectation to the solution of the following optimization problem:
\begin{align}\label{eq:optMod}
    w^o \eqdef \argmin_{w\in\mathbb{R}^M} \frac{1}{K}\sum_{k=1}^K q_k J_k(w).
\end{align}
We will establish this fact in subsequent sections where we examine how close the iterates at the agents get to this $w^o$. Here we provide some useful intuition to motivate \eqref{eq:optMod}. 

First, note that if an agent \( k \) does not participate in the learning process (i.e., \( q_k = 0\) throughout all iterations), then the optimization problem \eqref{eq:optMod}  should exclude the local loss function \( J_k(w) \). Additionally, suppose each agent \( k \) is initialized at \( w^o \) and consider the case where 
there are no local updates, i.e., \( T = 1 \). By taking the expectation of both sides of  recursion \eqref{diff-compact-local-update-partial}, we obtain:
\begin{align}
\mathbb{E}[\swb_{i+1}] &= \mathbb{E}\Big[\bm{\mathcal{A}}_{i+1} \big(\swb^o - \M \nabla \mathcal{Q}(\swb^o;\sxb_{{i}})\big)\Big] \nonumber \\
&= \bar{\mathcal{A}} \swb^o \hspace{-1mm}-\hspace{-1mm} \bar{\mathcal M} \nabla \mathcal{J}(\swb^o)\hspace{-1mm}+\hspace{-1mm}\mu\left( I-\bar{\mathcal A} \right) \nabla \mathcal{J}(\swb^o) \label{dift-optimal-solution}
\end{align}
where $\swb^o = w^o \otimes I_M$, $\bar{\mathcal M} = \bar{M} \otimes I_M$, and $\bar{\mathcal A} = \bar{A} \otimes I_M$. The last equality holds because $\bm{\mathcal{A}}_{i+1}$, $\M$ are independent of $\nabla \mathcal{Q}(\swb^o;\sxb_{{i}})$, and in view of relation \eqref{q23zand210}.
Left-multiplying both sides of \eqref{dift-optimal-solution} by \(\mathds{1}^\tran / K\), we obtain 
\begin{align}
\bar{w}_{i+1} \eqdef \frac{1}{K}\mathds{1}_K^\tran\mathbb{E}[\swb_{i+1}] = w^o - \frac{\mu}{K}\sum_{k=1}^K q_k \nabla J_k(w^o).
\end{align}
To guarantee that $\bar{w}_{i+1}$ stay at the optimal solution, i.e., $\bar{w}_{i+1} = w^o$, we need to ensure 
\begin{align}
    \frac{1}{K}\sum_{k=1}^K q_k \nabla J_k(w^o) = 0. 
\end{align}
This condition indicates that the \( \{\bm{w}_{k,iT}\} \) generated by Algorithm~\ref{alg:assynATC} converge in expectation to the solution of problem \eqref{eq:optMod}. See Figure~\ref{fig:drifts} for an illustration. 

{
\subsection{Drifts correction when activation probability is known}
Algorithm~\ref{alg:assynATC} converges in expectation to the solution of the drifted problem \eqref{eq:optMod}. However, these drifts can be corrected if the activation probabilities \( q_k \) are known.

Assuming each \( q_k > 0 \) is known to agent \( k \), we modify the step-size in \eqref{mu-agent-activation} as  
\begin{align}\label{mu-agent-activation2}
    \bm{\mu}^{\prime}_{k,iT+t} = \begin{cases}
        \mu/q_k ,& k \text{ active at iteration } i, \\
        0, & k \text{ inactive at iteration } i.
    \end{cases}
\end{align}  
Thus, an active agent \( k \) adopts a step-size inversely proportional to \( q_k \). Defining  
\begin{align}
    \bm{M}_{i}^{\prime} &\eqdef \mathrm {diag}\left( \bm{\mu}^{\prime}_{1,i},\dots,\bm{\mu}^{\prime}_{K,i} \right),\\
    \M^{\prime}&\eqdef \bm{M}^{\prime}_i \otimes I_M,
\end{align}  
Then the update rule \eqref{diff-compact-local-update-partial} becomes  
\begin{align}\label{diff-compact-local-update-partial2}
\swb_{iT+t} &= \A \Big(\swb_{iT+t-1} - \M^{\prime} \nabla \mathcal{Q}(\swb_{iT+t-1};\sxb_{{iT+t}})\Big), \nonumber \\
&\quad \forall t = 1,2,\dots, T, \quad \forall i=0,1,2,\dots.
\end{align}  
Now, suppose each agent starts at the fixed point \( w^{*} \) and set \( T=1 \). Taking expectations in \eqref{diff-compact-local-update-partial2} yields  
\begin{align}
\mathbb{E}[\swb_{i+1}] &= \mathbb{E}\Big[\bm{\mathcal{A}}_{i+1} \big(\swb^* - \M^{\prime} \nabla \mathcal{Q}(\swb^*;\sxb_{{i}})\big)\Big] \label{dift-optimal-solution2}\\
&= \bar{\mathcal{A}} \swb^* - \mu \nabla \mathcal{J}(\swb^*) + \mu^{2} (I-\bar{\mathcal A}) \bar{\mathcal M}^{-1} \nabla \mathcal{J}(\swb^*), \nonumber
\end{align}  
where \( \swb^* \eqdef w^* \otimes I_M \). The last equality follows from  
\begin{align}
\mathbb{E} \bm{\mathcal{A}}_{i+1} \M^{\prime} &= \mu^{2} (\Abar-I)\bar{\mathcal M}^{-1}+\mu I,
\end{align}  
which can be derived similarly to \eqref{q23zand210}. Left-multiplying \eqref{dift-optimal-solution2} by \( \mathds{1}^\tran / K \) gives  
\begin{align}
    \bar{w}_{i+1} = \frac{1}{K}\mathds{1}_K^\tran\mathbb{E}[\swb_{i+1}] = w^* - \frac{\mu}{K}\sum_{k=1}^K  \nabla J_k(w^*).
\end{align}  
To ensure \( \bar{w}_{i+1} \) remains at \( w^* \), we require  
\begin{align}
    \frac{1}{K}\sum_{k=1}^K  \nabla J_k(w^*) = 0.
\end{align}  
This condition implies that the fixed point $w^\star$ of the sequence \( \{\bm{w}_{k,iT}\} \) generated by \eqref{diff-compact-local-update-partial2} is the  solution of problem \eqref{eq:optProb}.  

\begin{figure}[t!]
    \centering
    \includegraphics[width=0.47\textwidth]{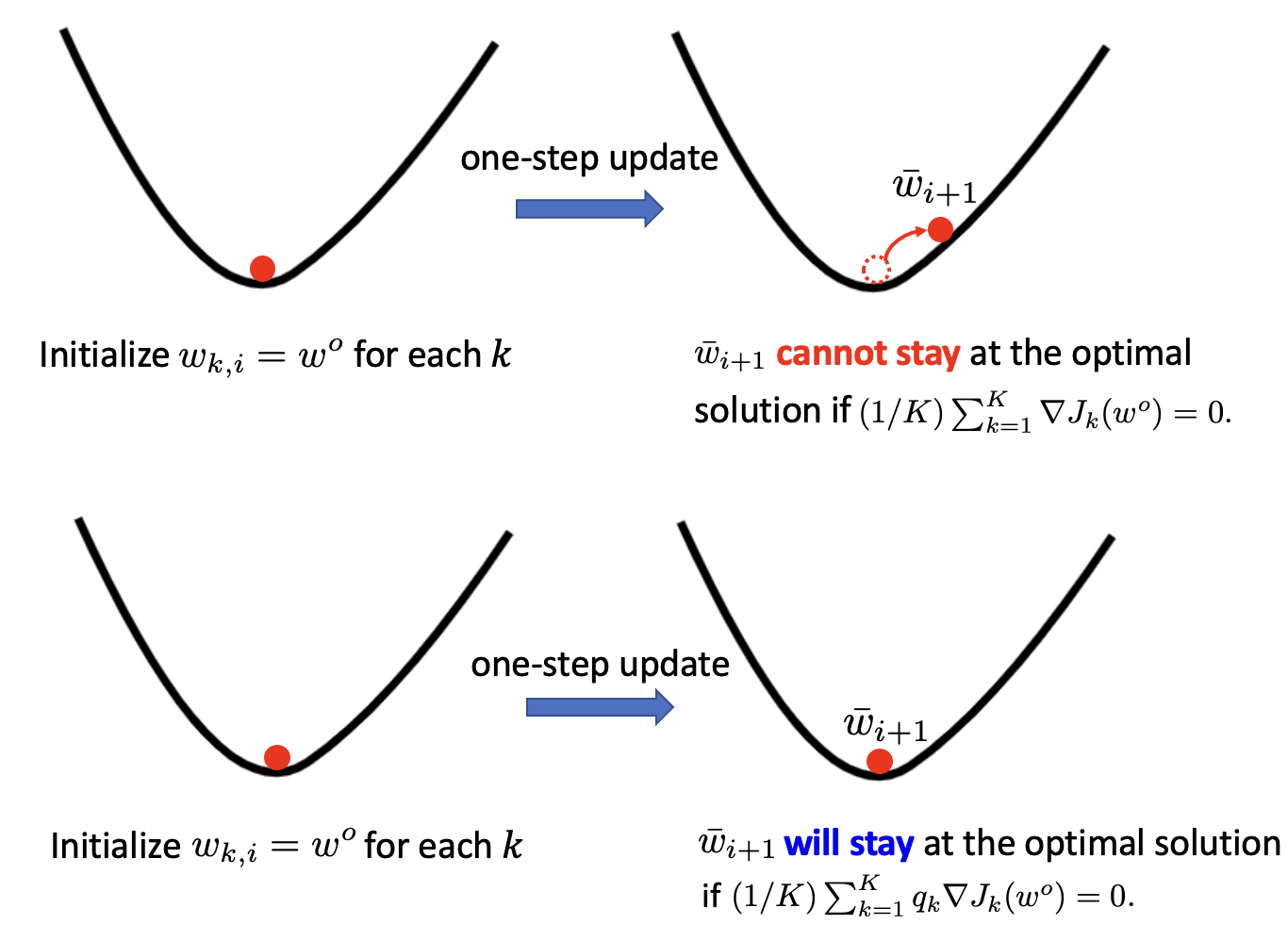}
    \caption{\small {Suppose each agent \( k \) starts at the optimal solution, i.e., \( w_{k,i} = w^o \). The one-step update \( \bar{w}_{i+1} \) will remain at the optimal solution only if the condition \( \frac{1}{K} \sum_{k=1}^K q_k \nabla J_k(w^o) = 0 \) is satisfied. This implies that \( w^o \) is, in fact, the optimal solution to problem \eqref{eq:optMod}. } 
    }
    \label{fig:drifts}
\end{figure}

Although scaling the step-sizes as in \eqref{mu-agent-activation2} is effective for correcting drifts, our convergence analysis will focus on the more prevalent scenario in which the activation probabilities \( q_k \) are not known a priori. In such cases, Algorithm~\ref{alg:assynATC} converges to the solution of the drifted problem \eqref{eq:optMod}. It is worth noting that the convergence results established in Sections~\ref{sec-sta} and \ref{sec-msd} can be readily extended to Algorithm~\eqref{diff-compact-local-update-partial2}.}

\section{Relation with existing algorithms}
Diffusion learning with local updates and partial agent participation, as described by \eqref{diff-compact-local-update-partial}, offers a versatile framework. By adjusting the number of local updates, the network topology, and the activation probability for each agent, the proposed recursion will reduce to several existing algorithms.

\vspace{1mm}
\noindent \textbf{Federated learning with full agent participation.} Let $q_k=1$ for each $k=1,\cdots, K$ and $\bm{A}_{iT} = (1/K)\mathds{1}_K\mathds{1}_K^\tran$, then recursion \eqref{diff-compact-local-update-partial} reduces to 
\begin{align}
\bm{\phi}_{k,iT + t} &= \bm{w}_{k,iT + t-1} - \mu \nabla_{w^\tran} Q_k(\bm{w}_{iT + t-1};\bm{x}_{k,iT+t}) \label{quadn-0}\\
\bm{w}_{k,iT + t} &=
\begin{cases}
  \bm{\phi}_{k,iT + t}, & \text{if $t\neq T$} \\
  (1/K)\sum_{\ell=1}^K \bm{\phi}_{\ell,iT + t}, & \text{if $t =  T$} \\
\end{cases} \label{quadn-1}
\end{align}
which is federated learning with full agent participation \cite{konevcny2016federated}. 

\vspace{1mm}
\noindent \textbf{Federated learning with partial agent participation.} Suppose a uniformly random subset $\mathcal{S}_i$ of agents is activated at each $i$-th block of iterations, with $|\mathcal{S}_i| = S$. In this scenario, if we set
\begin{align}
    \bm{a}_{\ell k, iT} = \begin{cases}
        \frac{1}{S}, & \text{ if both}\ \ell \:\&\: k \text{ active} \\
        1, & \text{ if $\ell = k$, and $k$ inactive}  \\
        0, & \text{ otherwise}
    \end{cases}
\end{align}
then recursion \eqref{diff-compact-local-update-partial} reduces to 
\begin{align}
\bm{\phi}_{k,iT \hspace{-0.1mm}+\hspace{-0.1mm} t} &\hspace{-1mm}=\hspace{-1mm} \bm{w}_{k,iT \hspace{-0.1mm}+\hspace{-0.1mm} t-1} \hspace{-1mm}-\hspace{-1mm} \mu \nabla Q(\bm{w}_{iT + t-1};\bm{x}_{k,iT+t}),\hspace{-0.5mm} \forall k \hspace{-0.8mm}\in\hspace{-0.8mm} \mathcal{S}_i \label{z7bnasd-0}\\
\bm{w}_{k,iT + t} &=
\begin{cases}
  \bm{\phi}_{k,iT + t} & \text{if $t\neq T$} \\
  (1/S)\sum_{\ell \in \mathcal{S}_i} \bm{\phi}_{\ell,iT + t} & \text{if $t =  T$} \label{z7bnasd-1} \\
\end{cases}
\end{align}
which is federated learning with partial participation \cite{karimireddy2020scaffold}.

\vspace{1mm}
\noindent \textbf{Standard diffusion learning.} Let $q_k=1$ for each agent $k=1,\cdots, K$, $T=1$, and $\bm{A}_{iT} = A$, recursion \eqref{diff-compact-local-update-partial} reduces to 
\begin{align}
\bm{\phi}_{k,i+1} &= \bm{w}_{k,i} - \mu \nabla_{w^\tran} Q_k(\bm{w}_{i};\bm{x}_{k,i}), \\
\bm{w}_{k,i+1} &=
  \sum_{\ell \in \mathcal{N}_k} a_{\ell k}\bm{\phi}_{\ell,i+1}, 
\end{align}
which is standard vanilla diffusion learning method \cite{sayed2014adaptation,sayed2014adaptive}.

\vspace{1mm}
\noindent \textbf{Asynchronous diffusion learning.} Let 
each agent $k$ get activated with probability $q_k$ and set $T=1$. If the learning rate $\bm{\mu}_{k,iT+t}$ follows \eqref{mu-agent-activation} and the combination matrix $\bm{A}_{iT}$ follows \eqref{z71}, then recursion \eqref{diff-compact-local-update-partial} reduces to 
\begin{align}
\bm{\phi}_{k,i+1} &= \bm{w}_{k,i} - \bm{\mu}_{k,i} \nabla_{w^\tran} Q_k(\bm{w}_{i};\bm{x}_{k,i}), \\
\bm{w}_{k,i+1} &=
  \sum_{\ell \in \mathcal{N}_k} \bm{a}_{\ell k,i}\bm{\phi}_{\ell,i+1}, 
\end{align}
which is the asynchronous diffusion learning method \cite{zhao2014asynchronous}.

\vspace{1mm}
\noindent \textbf{Decentralized federated learning.} Let 
$q_k=1$, $\bm{A}_{iT}=A$, and $\bm{\mu}_{k,iT+t}$ follows \eqref{mu-agent-activation}, then recursion \eqref{diff-compact-local-update-partial} reduces to 
\begin{align}
\bm{\phi}_{k,iT + t} &= \bm{w}_{k,iT + t-1} - \mu \nabla_{w^\tran} Q_k(\bm{w}_{iT + t-1};\bm{x}_{k,iT+t}) \label{va8-0} \\
\bm{w}_{k,iT + t} &=
\begin{cases}
  \bm{\phi}_{k,iT + t} & \text{if $t\neq T$} \\
  \sum_{\ell \in \mathcal{N}_k} a_{\ell k} \bm{\phi}_{\ell,iT + t} & \text{if $t =  T$} \\
\end{cases}\label{va8-1}
\end{align}
which is decentralized federated learning  \cite{sun2022decentralized}.

\vspace{1mm}
Due to its versatile algorithmic structure, the convergence analysis—and, in particular, the mean-square-deviation (MSD) performance—of Algorithm~\ref{alg:assynATC} directly applies to the aforementioned algorithms. To our knowledge, the tight MSD expression for federated learning (i.e., recursions \eqref{quadn-0}-\eqref{quadn-1} and recursions \eqref{z7bnasd-0}-\eqref{z7bnasd-1}) and decentralized federated learning (i.e., recursions \eqref{va8-0}-\eqref{va8-1}) had not been established earlier in the literature.

\section{Stability Analysis}
\label{sec-sta}
We begin our analysis by establishing the mean-square-error stability of the algorithm, and by showing that the agents succeed in estimating the true model $w^o$ characterized by \eqref{eq:optMod}. This process begins with defining the gradient noise and then establishing the error recursion formulation. Thus, we define the errors as $\we_{k,iT+t} \eqdef w^o - \w_{k,iT+t}$, $\wse_{iT+t} \eqdef \col{\we_{k,iT+t}}$, the local gradient noise as
\begin{align}
    &\quad \quad \bm{s}_{k,iT+t} \nonumber \\
    &\eqdef \nabla_{w^\tran} Q_k(\bm{w}_{iT + t-1};\bm{x}_{k,iT+t})  \hspace{-0.5mm} - \hspace{-0.5mm}  \grad{w}J_k(\w_{k,iT+t-1}) 
    ,
\end{align}
and the global gradient noise as the stacked column of individual gradient noises:
\begin{align}
    \bm{s}_{iT+t} \eqdef \col{\bm{s}_{k,iT+t}}.
\end{align}
We now verify that the estimate of the gradient is unbiased and that the variance of the gradient noise is bounded.
\begin{lemma}[\sc{Second-order Gradient noise}]\label{lemm:gradNoise}
    The individual gradient noise $\bm{s}_{k,iT+t}$ is zero-mean and has bounded second-order moment, namely, there exist constants $\beta_s^2$ and $\sigma_s^2$ such that:
	\begin{align}
		\mathbb{E}\Vert \bm{s}_{k,iT+t}\Vert^2  \leq \beta_s^2 \mathbb{E}\Vert \we_{k,iT+t-1}\Vert^2 + \sigma_s^2,
	\end{align}
    Consequently, the network noise $\bm{s}_{iT+t}$ is also zero-mean and has bounded second-order moment:
	 	\begin{align}
	 		\mathbb{E}\Vert \bm{s}_{iT+t}\Vert^2  \leq \beta_s^2 \mathbb{E}\Vert \wse_{iT+t-1}\Vert^2 + \sigma_s^2.
	 	\end{align}
\end{lemma}
\begin{proof}
The proof is found in Appendix \ref{app:lemmGradNoise}.
\end{proof}
\begin{lemma}[\sc{Fourth-order Gradient noise}]\label{lemm:gradNoise4}
    The individual gradient noise $\bm{s}_{k,iT+t}$ is zero-mean and has bounded fourth-order moment, namely, there exist constants $\bar{\beta}_s^4$ and $\bar{\sigma}_s^4$ such that:
	\begin{align}
		\mathbb{E}\Vert \bm{s}_{k,iT+t}\Vert^4  \leq \bar{\beta}_s^4 \mathbb{E}\Vert \we_{k,iT+t-1}\Vert^4 + \bar{\sigma}_s^4.
	\end{align}
	Consequently, the network noise $\bm{s}_{iT+t}$ is also zero-mean and has bounded fourth-order moment:
	 	\begin{align}
	 		\mathbb{E}\Vert \bm{s}_{iT+t}\Vert^4  \leq \bar{\beta}_s^2 \mathbb{E}\Vert \wse_{iT+t-1}\Vert^4 + \bar{\sigma}_s^4.
	 	\end{align}
\end{lemma}
\begin{proof}
The proof is found in Appendix \ref{app:lemmGradNoise4}.
\end{proof}
Next, we write an error recursion in terms of the gradient noise variables. We first introduce the following terms:
\begin{align}
    \Hi &\eqdef \col{H_{k,iT+t-1}}_{k=1}^K, \\
	H_{k,iT+t-1} & \eqdef \int_0^1 \nabla_{w}^2 J_k(w^o-\tau \tilde{\w}_{k,iT+t-1} ) d\tau, \\
	b &\eqdef -\col{\grad{w}J_k(w^o)},
\end{align}
and use them to write:
\begin{align}\label{eq:errRec}
	\wse_{iT+t} = &\A (I- \M \Hi )\wse_{iT+t-1} - \A \M  b
 \notag \\ &
 + \A \M \bm{s}_{iT+t}.
\end{align}

In the theorem below, we establish that the algorithm converges to a region around $w^o$.

\begin{theorem}[\sc{Second-order stability}]\label{thrm:2ndStab}
Under Assumptions \ref{assum:combMat} and \ref{assum:fct}, and for small enough step-size $\mu$, it holds that for every agent $k=1,2,\cdots, K$ that 
\begin{align}
    \limsup_{i\to \infty} \mathbb{E} \Vert {\we_{k,iT}}\Vert^2 = O(\mu).
\end{align} 
\end{theorem}
\begin{proof}
    The proof is found in Appendix \ref{app:thrm2ndStab}.
\end{proof}
While the algorithm remains stable, the effects of its time-varying nature are captured in its convergence rate and error magnitude. Specifically, the convergence rate increases for larger activation probabilities and number of local steps. However, the performance deteriorates with a larger number of local steps.

Before arriving at performance expressions, we need to establish the fourth-order stability of the algorithm. This result will be used to introduce a long-term model and use it as a good substitute for the original model. 

\begin{theorem}[\sc{Fourth-order stability}]\label{thrm:4thStab}
    Under Assumptions \ref{assum:combMat} and \ref{assum:fct}, and for small enough step-size $\mu$, it holds that for every agent $k=1,2,\cdots, K$:
	\begin{align}
		\limsup_{i\to \infty} \mathbb{E} \Vert {\we_{k,iT}}\Vert^4 = O(\mu^2).
 	\end{align}  
\end{theorem}
\begin{proof}
    The proof is found in Appendix \ref{app:thrm4thStab}.
\end{proof}

\section{Performance Analysis}
\label{sec-msd}
{Having established the mean-square-error stability of the proposed Algorithm \ref{alg:assynATC}, we now proceed to derive an expression for its steady-state mean-square-deviation (MSD) performance. This analysis will enable us to evaluate how the hyperparameters of the proposed algorithm—specifically, the number of local updates \(T\) and the participation probability \(q_k\)—influence its efficacy in learning the model. We will further extend the derived MSD expression to other federated learning algorithms, providing precise estimates of their steady-state error magnitudes.}


{
\subsection{Mean-square-deviation Metric}
In Theorem \ref{thrm:2ndStab}, we establish that Algorithm \ref{alg:assynATC} converges with sufficiently small step-size $\mu$, and the size of the steady-state error is on the order of $O(\mu)$. We now pursue a closed-form expression for algorithmic performance by using mean-square-deviation (MSD) defined as follows \cite[Chapter 11]{sayed2014adaptation}: 
\begin{align}\label{eq:MSD-expr}
    \mathrm{MSD} \eqdef \mu \Big( \lim_{\mu \to 0} \limsup_{i\to \infty} \frac{1}{\mu K} \sum_{k=1}^K \mathbb{E}\|\we_{k,iT}\|^2 \Big).
\end{align}
This MSD expression provides a precise estimate of the steady-state error, rather than just its order. While the MSD expression is well-established for classical decentralized learning algorithms—such as diffusion \cite{chen2012diffusion, sayed2014adaptive, sayed2014adaptation}, asynchronous diffusion \cite{zhao2014asynchronous}, and exact diffusion \cite{yuan2020influence}—it has not yet been developed for federated learning algorithms that incorporate local updates within their framework. In this way, the analysis in this paper leads to what is apparently the first MSD expression for learning algorithms that incorporate local updates and partial agent participation.  }

\subsection{Long-term Model}
We start by showing that it is sufficient to introduce and examine the evolution of a long-term model in place of the original recursion \eqref{eq:errRec}. For this purpose, we first introduce the error Hessian matrix:
\begin{align}
    \widetilde{\bm{\mathcal{H}}}_{iT+t-1} &\eqdef \mathcal{H} - \Hi, \\
    \mathcal{H} &\eqdef \diag{H_k}, \\
    H_k &\eqdef \nabla_{w}^2 J_k(w^o),
\end{align}
and rewrite the original error recursion \eqref{eq:errRec} as:
\begin{align}
    \wse_{iT+t} = & \A (I - \M \mathcal{H}) \wse_{iT+t-1} - \A \M b 
    \notag \\
    & + \A \M \bm{s}_{iT+t} + \A \M \bm{c}_{iT+t-1},
\end{align}
where we are introducing the term:
\begin{align}
    \bm{c}_{iT+t-1} \eqdef \widetilde{\bm{\mathcal{H}}}_{iT+t-1}  \wse_{iT+t-1}.
\end{align}
By using Assumption \ref{assum:hess}, we have:
\begin{align}
    \Vert \widetilde{\bm{\mathcal{H}}}_{iT+t-1}\Vert \leq \frac{ \kappa}{2}\Vert \wse_{iT+t-1}\Vert,
\end{align}
and using Theorem \ref{thrm:2ndStab} we can verify that:
\begin{align}
    \limsup_{i\to\infty} \mathbb{E} \Vert \bm{c}_{iT+t-1}\Vert \leq O(\mu).
\end{align}
As such, we can use the following long-term model to approximate the dynamics of the original recursion \eqref{eq:errRec}:
\begin{align}\label{eq:longTermRec}
    \wse'_{iT+t} = & \A^\tran (I - \M \mathcal{H}) \wse'_{iT+t-1} - \A^\tran \M b 
    \notag \\
    & + \A^\tran \M \bm{s}_{iT+t} 
\end{align}
This model differs from \eqref{eq:errRec} in that it replaces $\bm{\mathcal{H}}_{iT+t-1}$ by the constant matrix $\mathcal{H}$. The following result justifies our reliance on \eqref{eq:longTermRec} to establish the MSD expression.
\begin{theorem}[\sc{Approximation error}]\label{thrm:approxErr}
The trajectories of the original error recursion \eqref{eq:errRec} and the long-term model \eqref{eq:longTermRec} are asymptotically close for small enough step-size $\mu$ and under Assumption \ref{assum:hess}, namely:
\begin{align}
    &\limsup_{i\to\infty} \mathbb{E} \Vert \wse_{iT+t} - \wse'_{iT+t}\Vert^2 = O(\mu^2), \\
    &\limsup_{i \to \infty} \mathbb{E} \Vert \wse_{iT+t}\Vert^2 \hspace{-0.5mm}=\hspace{-0.5mm}\limsup_{i \to \infty} \mathbb{E} \Vert \wse'_{iT+t}\Vert^2 \hspace{-0.5mm}+\hspace{-0.5mm} O(\mu^{3/2}).
\end{align}
\end{theorem}
\begin{proof}
    The proof can be found in Appendix \ref{app:thrmApproxErr}.
\end{proof}
We next show that the long-term model is stable. 
\begin{theorem}[\sc{Long-term model stability}]\label{thrm:stabLongTerm}
    Under Assumptions \ref{assum:combMat} and \ref{assum:fct}, and for a small enough step-size $\mu$, it holds for every agent $k=1,2,\cdots, K$:
    \begin{align}
        \limsup_{i\to\infty} \mathbb{E}\Vert \we'_{k,iT+t}\Vert^2 = O(\mu).
    \end{align}
\end{theorem}
\begin{proof}
    The proof can be found in Appendix \ref{app:thrmStabLongTerm}.
\end{proof}

\subsection{MSD Expression}
For the subsequent discussion, we require the following assumption on the gradient process. 
\begin{assumption}[\sc{Noise process}]\label{assum:noiseProc}
    Define the covariance of the gradient noise:
    \begin{align}
        R_{k,iT+t} (w) \eqdef \mathbb{E} \bm{s}_{k,iT+t}(w)\bm{s}^\tran_{k,iT+t}(w).
    \end{align}
    Then, for some positive constant $\kappa_s$ and $\alpha_s$, the covariance satisfies the following Lipschitz condition: 
    \begin{align}
        \Vert \mbox{\rm diag} \left\{R_{k,iT+t}(w^o)-R_{k,iT+t}(\w_{k,iT+t}) \right\} \Vert \leq \kappa_s \Vert \wse_{iT+t}\Vert^{\alpha_s}. 
    \end{align}
    Moreover, the following limit exists:
    \begin{align}
        R_k \eqdef \lim_{i\to \infty}R_{k,iT+t}(w^o). 
    \end{align}
    \qed
\end{assumption}

With this assumption in addition to the smoothness assumption on the Hessian, we can find an expression for the MSD.

\begin{theorem}[\sc{Steady-state MSD}]\label{thrm:MSD}
    Under Assumptions \ref{assum:hess} and \ref{assum:noiseProc}, it holds that:
    \begin{align}
        \mbox{\rm MSD} &\eqdef \lim_{i\to \infty}\frac{1}{K}\sum_{k=1}^K \mathbb{E} \Vert \we_{k,(i+1)T}\Vert^2 
    \notag \\
    &= \frac{1}{K}z^\tran {\rm bvec}(I_{MK}) + O(\mu^{\alpha}),  \label{eq:MSD-expre}
    \end{align}
    where $\alpha \eqdef 1 + \frac{1}{2}\min \{1,\alpha_s\}$ and $z$ is defined in \eqref{eq:defZ} in terms of the gradient noise and the bias; it is also approximated in \eqref{zuqn} and \eqref{b2987a}.
\end{theorem}
\begin{proof}
    The proof can be found in Appendix \ref{app:thrmMSD}.
\end{proof}
\begin{remark}[\sc Influence of local updates and activation probability]
By examining the expression for $z$ in \eqref{eq:defZ} we can see that it is proportional to 
\begin{align}\label{zuqn}
   z &\propto \mu T (I-\mu \mathcal{H})^{2T-2} \Big( b\otimes_b b  + {\rm bvec}\left(\diag{R_k} \right) \Big)
   \notag \\ &\quad 
   + O(\mu) (I - \mu \mathcal{H})^{2T-1}.
\end{align}
Thus, for larger $T$, $z$ increases resulting in larger MSD. Additionally, by setting $T=1$, we can see that $z$ is proportional to the activation probabilities:
\begin{align}\label{b2987a}
    z &\propto \left( I - (1-O(\mu))^2\mathbb{E} \bm{\mathcal{A}}_i \otimes \bm{\mathcal{A}}_i \right)^{-1} O(\mu^2) \mathbb{E} \bm{\mathcal{A}}_i \otimes \bm{\mathcal{A}}_i 
    \notag \\ &\quad \times
    \left( b\otimes_b b + {\rm bvec}(\diag{R_k}) \right).
\end{align}
Additionally, the gradient noise decreases as the activation probabilities increase, resulting in a smaller MSD.
\end{remark}

\section{Experimental Analysis}
We consider a network of $K=20$ agents (Fig. \ref{fig:Net}) whose goal is to solve a linear regression problem. Every agent $k$ has a dataset that consists of $N = 100$ input vectors $\bm{u}_{k,n} \in \mathbb{R}^2$ that are normally distributed with covariance matrix $R_{u}$ and varying mean. We assume the output $\bm{d}_k(n)$ satisfies the following relation:
\begin{align}
    \bm{d}_{k}(n) = \bm{u}_{k,n}^\tran w^{\star} + \bm{v}_{k}(n),
\end{align}
where $w^{\star} \in \mathbb{R}^2$ is some generative model and $\bm{v}_{k}(n)$ is random zero-mean Gaussian noise with $\sigma_{k,v}^2$ variance and is independent of the input $\bm{u}_{k,n}$. To ensure non-IID data, we enforce varying input means and noise variances. We define the following regularized optimization problem:
\begin{align}
    \min_{w\in\mathbb{R}^2} \frac{1}{K N} \sum_{k=1}^K \sum_{n=1}^K \vert \bm{d}_{k}(n) - \bm{u}_{k,n}^\tran w \vert^2 + \rho \Vert w\Vert^2. 
\end{align}
\begin{figure}[t!]
    \centering
    \includegraphics[width=0.3\textwidth]{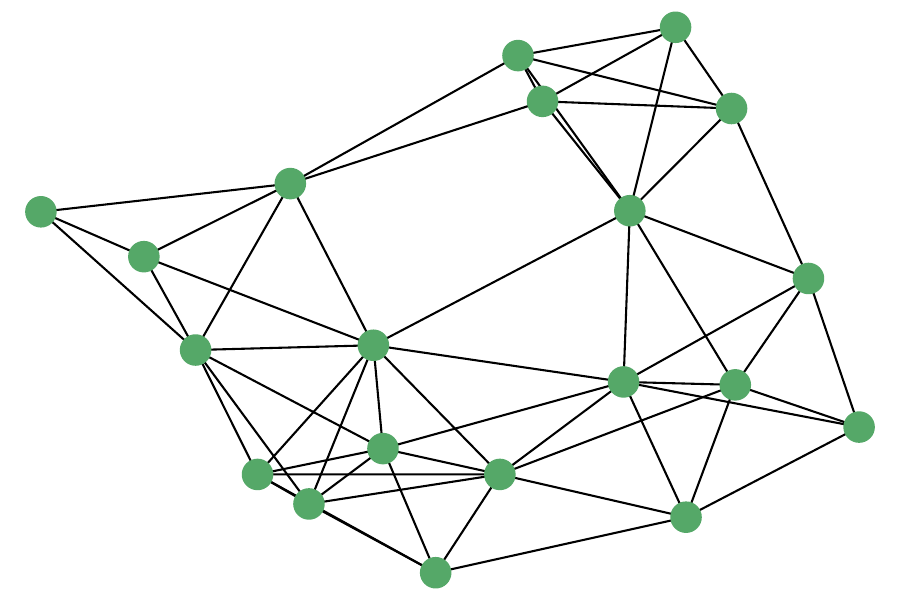}
    \caption{The underlying network.}
    \label{fig:Net}
\end{figure}

We first run Algorithm \ref{alg:assynATC} with local updates and partial agent participation. We set the step-size to $\mu = 0.01$ and the regularization parameter to $\rho = 0.1$. We generate random participation probabilities $q_k$ and set the number of local updates to $T = 5$. We repeat the algorithm 5 times. We calculate the average MSD of the individual agents over the total number of agents and the 5 passes.
The resulting learning curve is plotted in Fig. \ref{fig:MSD-nonIID}, and we observer that it matches the theoretical MSD expression established in Theorem \ref{thrm:MSD}. 
\begin{figure}[h!]
    \centering
    \includegraphics[width=0.4\textwidth]{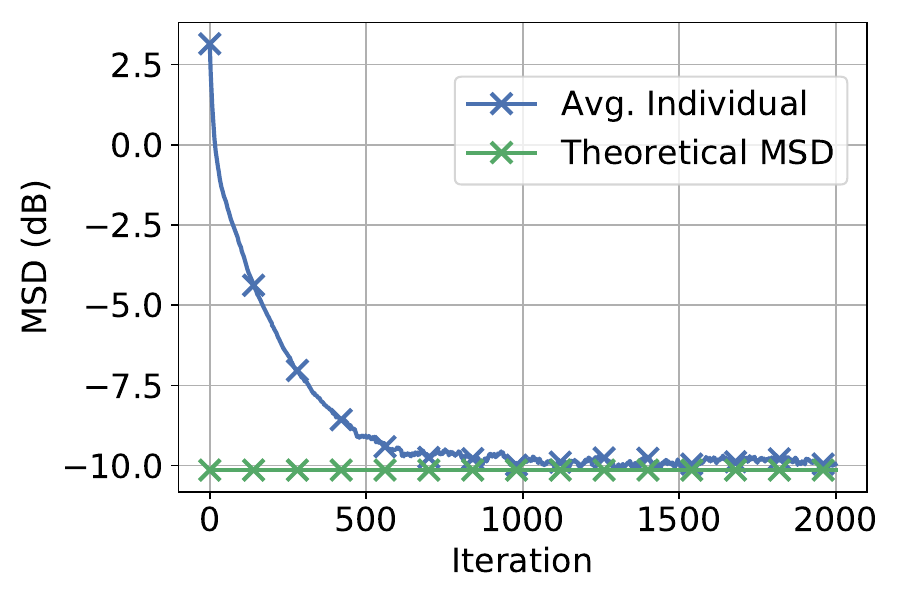}
    \caption{The steady-state performance of Algorithm \ref{alg:assynATC} with local updates and partial agent participation matches with the theoretical MSD expression \eqref{eq:MSD-expre} established in Theorem \ref{thrm:MSD}.}
    \label{fig:MSD-nonIID}
\end{figure}

\subsection{Effect of Partial Agent Participation}
We assume agents run only one local update and vary the activation probabilities to study its effect on algorithm performance. We set all the probabilities $q_k$ to the same value across the agents. We test three cases: a low activation case when $q_k = 0.1$, a medium activation case when $q_k = 0.5$, and a high activation case when $q_k = 0.9$. We plot the average MSD for these three scenarios in Fig. \ref{fig:MSD-prob}. We observe that both convergence rate and performance improve as the activation probabilities increase. This is expected, as higher agent participation probabilities lead to more active agents during a given iteration, resulting in a better gradient estimate. Furthermore, a better gradient estimate during an iteration leads to a smaller converging neighborhood, influenced by the gradient noise variance.
\begin{figure}[h!]
    \centering
    \includegraphics[width=0.4\textwidth]{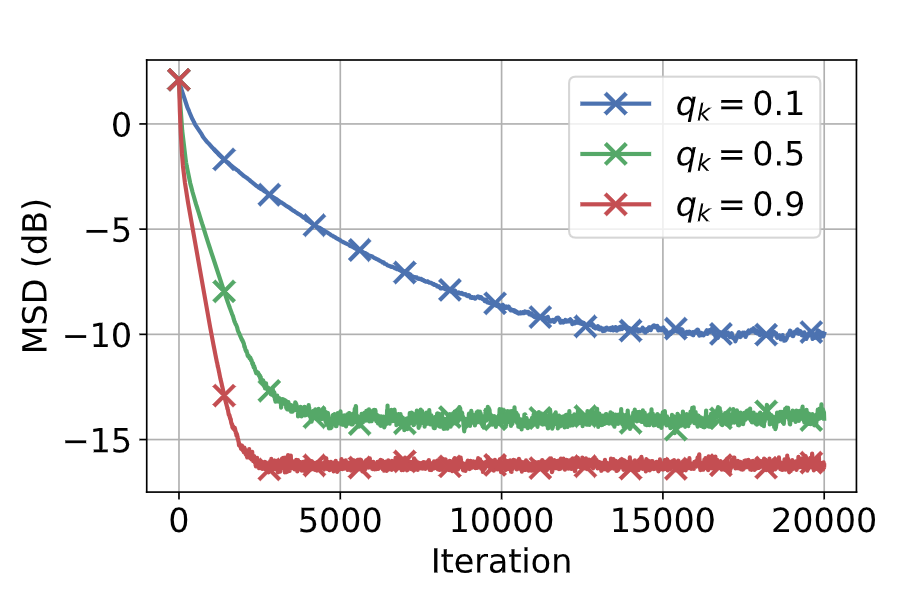}
    \caption{The convergence performance of Algorithm \ref{alg:assynATC} with varying agent activation probabilities.}
    \label{fig:MSD-prob}
\end{figure}

\subsection{Effect of Local Updates}
We assume all agents are active but study the effect of the local updates. Thus, we consider three values of local updates, $T = 2,5,10$. We plot the corresponding MSD curves in Fig. \ref{fig:MSD-locUp}. They converge faster to a worse error as agents perform more local updates. Studying closely the expressions in Theorem \ref{thrm:2ndStab}, we can verify this behavior. As $T$ increases, the bounds on the bias and gradient noise are proportional to $T$ and thus they increase. Furthermore, the rate of convergence is driven by $1-T\sum_{k}q_k O(\mu)$, and therefore as $T$ increases the rate approaches zero resulting in a faster convergence. 

\begin{figure}[h!]
    \centering
    \includegraphics[width=0.4\textwidth]{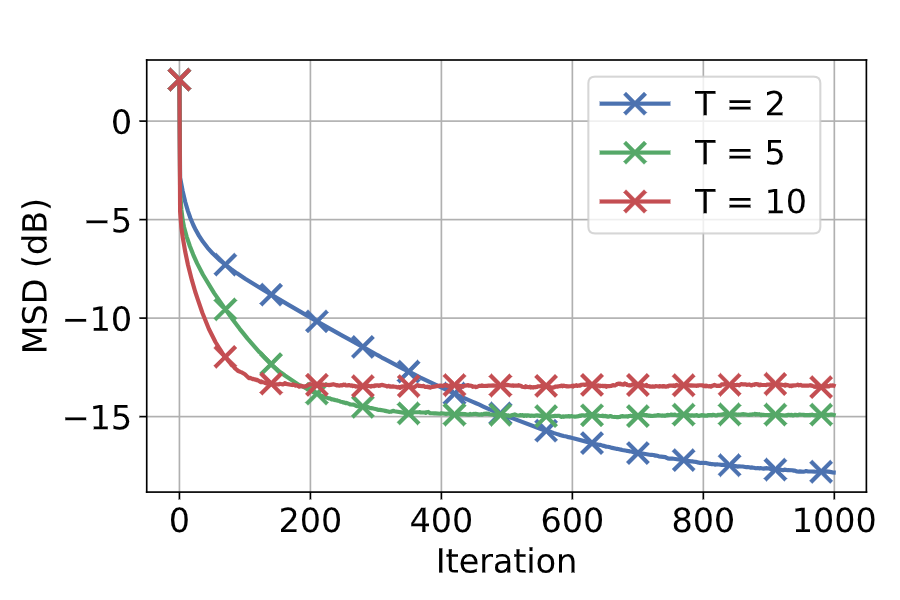}
    \caption{The convergence performance of Algorithm \ref{alg:assynATC} with varying number of local updates.}
    \label{fig:MSD-locUp}
\end{figure}

\section{Conclusion}
This paper investigates a diffusion learning strategy that incorporates local updates and partial agent participation, addressing scenarios with communication constraints and agent unavailability. We demonstrate that the proposed algorithm maintains mean-square error stability and offer a thorough analysis of its MSD performance. Through our numerical experiments, we show that local steps hinder the performance of the algorithm, while higher activation probabilities lead to a speed-up in convergence to a better approximation of the model. 

\appendices

\section{Proof of Lemma \ref{lemm:topo}}\label{app:lemmtopo}
{We first calculate $\bar{A}_{iT+t}$. Rewrite $\bm{a}_{\ell k,iT+t}$ as
\begin{align}
\bm{a}_{\ell k,iT+t} = 
\begin{cases}
a_{\ell k}\cdot\mathbb{I}(\text{\textit{l} is active})\cdot\mathbb{I}(\text{\textit{k} is active}),
& \ell\neq k,\\
1-\sum\limits_{j\neq k}\bm{a}_{\ell j,iT+t},&\ell = k.
\end{cases}
\end{align}
where the operator $\mathbb I(\cdot)$ on a given event $X$ is defined as
\begin{align}
    \mathbb{I}(X)=
    \begin{cases}
    1,& \text{when } X \mbox{ is true}\\
    0,& \text{o.w.}
    \end{cases}
\end{align}
Since agents are independently active with probability $q_{k}$,
\begin{align}
\bar{a}_{\ell k,iT+t}=
\begin{cases}
a_{\ell k}q_{\ell}q_{k},&\ell \neq k,\\
1-\sum\limits_{j\neq k} a_{\ell k }q_{\ell}q_{k},&\ell = k.
\end{cases}
\end{align}
Next, we calculate $\mathbb E \bm{a}_{\ell k,iT+t}\bm{\mu}_{k,i}$. Rewrite $\bm{a}_{\ell k,iT+t}\bm{\mu}_{k,i}$ as 
\begin{align}
&\bm{a}_{\ell k,iT+t}\bm{\mu}_{i}\notag\\
&=
\begin{cases}
a_{\ell k}\mu \mathbb{I}(\text{\textit{l} is active})\cdot\mathbb{I}(\text{\textit{k} is active}),&\ell \neq k\\
\left[ 1-\sum\limits_{j\neq k}a_{\ell j}\mathbb I(\text{\textit{j} is active}) \right]\mu\mathbb I(\text{\textit{k} is active}),&\ell = k.
\end{cases}
\end{align}
Taking expectations, we have:
\begin{align}
\mathbb E\bm{a}_{\ell k,iT+t}\bm{\mu}_{i}
&=
\begin{cases}
a_{\ell k}\mu q_{\ell}q_{k},&\ell \neq k\\
\mu\bar{a}_{kk.iT+t}-\mu(1-q_{k}),&\ell = k.
\end{cases}
\end{align}
With the notation of \eqref{z72bvba}, we arrive at
\begin{align}
\mathbb{E} \A \M &= \mu (\Abar_{iT+t}-I)+\bar{\mathcal M}.
\end{align}
}

\section{Proof of Lemma \ref{lemm:gradNoise}} \label{app:lemmGradNoise}
	We start by showing that the individual gradient noise is zero-mean. If we assume the gradient estimate is some mini-batch estimate, then it is simple to show that the conditional expectation of the gradient noise is zero, which also implies that $\mathbb{E} \{\bm{s}_{iT+t} | \F\} = 0$.

	Next, to show that the second-order moment is bounded, we bound the $\ell_2-$norm by using Jensen's inequality and the Lipschitz condition:
    \begin{align}
    \left\lVert \bm{s}_{k,iT+t} \right\rVert^{2}
    &\leq
    3\left \lVert \grad{w}J_k(\w_{k,iT+t-1}) -  \grad{w}J_k(w^o)\right\rVert^{2} \nonumber 
			 \\
			&\quad +  3\left\lVert \grad{w}J_k(w^o) -  \grad{w}Q_k(w^o;\x_{k,iT+t}) \right\rVert^{2}
					\nonumber	 \\
			&\quad + 3\Big\lVert \grad{w}Q_k(w^o;\x_{k,iT+t})  \nonumber \\
            &\quad \quad- \grad{w}Q_k(\w_{k,iT+t-1};\x_{k,iT+t}) \Big\rVert^{2} \nonumber  \\
    &\leq 6\delta^2 \left\lVert \we_{k,iT+t-1}\right\rVert^2 \nonumber \\
    &\quad+3 \left\lVert  \grad{w}Q_k(w^o;\x_{k,iT+t}) \right\rVert^{2}.
    \end{align}
	Then, by defining the following constants:
	\begin{align}
		\beta_s^2 &\eqdef {6} \delta^2, \\
		\sigma_s^2 &\eqdef {3 \sum_{k=1}^K \mathbb{E} \left\{ \left\lVert  \grad{w}Q_k(w^o;\x_{k,iT+t}) \right\rVert^2  \right\}},
	\end{align}
	we can conclude bounded second-order moment of the individual gradient noise as well as the total gradient noise since:
    {
    \begin{align}
            &\mathbb E\left\lVert \bm{s}_{iT+t} \right\rVert^{2} = \sum\limits_{k=1}^{K}\mathbb E\left\lVert \bm{s}_{k,iT+t} \right\rVert^{2} \nonumber \\
            &\leq \sum\limits_{k=1}^{K}\left(6\delta^{2}\mathbb E\left\lVert \we_{k,iT+t-1} \right\rVert^{2} \hspace{-0.5mm}+\hspace{-0.5mm} 3\mathbb{E} \Vert\grad{w}Q_k(w^o;\x_{k,iT+t})\Vert^2 \right) \nonumber \\
            &= \beta_s^2 \mathbb E\left\Vert \wse_{iT+t-1}\right\Vert^2 + \sigma_s^2.
        \end{align}
    }
\allowdisplaybreaks
\section{Proof of Lemma \ref{lemm:gradNoise4}} \label{app:lemmGradNoise4}
Similar to the proof of Lemma \ref{lemm:gradNoise}, we bound the $\ell_4-$norm by using Jensen's inequality and the Lipschitz condition:
\begin{align}
    \left\lVert \bm{s}_{k,iT+t} \right\rVert^{4}
    &\leq
    64\left \lVert \grad{w}J_k(\w_{k,iT+t-1}) -  \grad{w}J_k(w^o)\right\rVert^{4}
		\nonumber	 \\
			&\quad +  64\left\lVert \grad{w}J_k(w^o) -  \grad{w}Q_k(w^o;\x_{k,iT+t}) \right\rVert^{4}
					\nonumber	 \\
			&\quad + 64\Big\lVert \grad{w}Q_k(w^o;\x_{k,iT+t})\Big\rVert^{4} \nonumber \\
            &\quad + 64\Big\lVert\grad{w}Q_k(\w_{k,iT+t-1};\x_{k,iT+t}) \Big\rVert^{4} \nonumber \\
    &\leq 128\delta^4 \left\lVert \we_{k,iT+t-1}\right\rVert^4 \nonumber \\
    &\quad + 64\Big(\left\lVert  \grad{w}Q_k(w^o;\x_{iT+t}) \right\rVert^{4} \nonumber \\
    &\quad + \left\lVert  \grad{w}Q_k(\w_{k,iT+t-1};\x_{iT+t}) \right\rVert^{4}\Big)
    .
    \end{align}
    Then, by defining the following constants:
	\begin{align}
		\bar{\beta}_s^2 &\eqdef 128K \delta^4, \\
		\sigma_s^2 &\eqdef 64K \sum_{k=1}^K\Big( \mathbb{E} \left\{ \left\lVert  \grad{w}Q_k(w^o;\x_{k,iT+t}) \right\rVert^4  \right\} \nonumber  \\
        &\quad \:+\:\mathbb E\left\lVert  \grad{w}Q_k(\w_{k,iT+t-1};\x_{k,iT+t}) \right\rVert^{4}\Big),
	\end{align}
we can conclude bounded fourth-order moments for the individual gradient noise, as well as for the total gradient noise since:
\begin{align}
    \mathbb E\left\lVert \bm{s}_{iT+t} \right\rVert^{4} & = \mathbb E\left( 
    \sum\limits_{k=1}^{K}\left\lVert \bm{s}_{k,iT+t} \right\rVert^{2}\right)^{2} \nonumber \\
    &\leq K \sum\limits_{k=1}^{K}\mathbb E\left\lVert \bm{s}_{k,iT+t} \right\rVert^{4}\nonumber
    \\
    &\leq 128K\delta^{4}\sum\limits_{k=1}^{K}
    \mathbb E\left\lVert \we_{k,iT+t-1} \right\rVert^{4}\nonumber\\
    &\quad +64K \sum_{k=1}^K\Big( \mathbb{E} \left\{ \left\lVert  \grad{w}Q_k(w^o;\x_{k,iT+t}) \right\rVert^4  \right\}\nonumber\\
    &\quad\quad +\mathbb E\left\lVert  \grad{w}Q_k(\w_{k,iT+t-1};\x_{k,iT+t}) \right\rVert^{4}\Big)\nonumber\\
    &\leq \bar{\beta}_s^4 \mathbb E\left\Vert \wse_{iT+t-1}\right\Vert^4 + \bar{\sigma}_s^4.
\end{align}

\section{Proof of Theorem \ref{thrm:2ndStab}}\label{app:thrm2ndStab}
We start with the network error recursion after all the local steps were conducted during iteration $i$: 
	\begin{align}
		\wse_{(i+1)T} = &\bm{\mathcal{A}}_{(i+1)T}^\tran (I - \M\bm{\mathcal{H}}_{(i+1)T-1}) \wse_{(i+1)T-1}
    \notag \\ &   
  - \bm{\mathcal{A}}_{(i+1)T}^\tran \M b + \bm{\mathcal{A}}_{(i+1)T}^\tran \M \bm{s}_{(i+1)T},
	\end{align}
    and by recursively replacing the error recursion for $t\neq T$:
    \begin{align}
        \wse_{iT+t} = & (I-\M \Hi) \wse_{iT+t-1} - \M b + \M \bm{s}_{iT+t},
    \end{align}
    we can write an error recursion that depends on the global iteration $i$:
    \begin{align}
        &\wse_{(i+1)T} =  \bm{\mathcal{A}}_{(i+1)T}^\tran \prod_{t=1}^T(I - \M\bm{\mathcal{H}}_{(i+1)T-t}) \wse_{iT} 
        \notag \\
        & - \bm{\mathcal{A}}_{(i+1)T}^\tran \left \{ I + \sum_{t=1}^{T-1} \prod_{t'=1}^t (I - \M \bm{\mathcal{H}}_{(i+1)T-t'}  )\right\}\M b
        \notag \\
        & + \bm{\mathcal{A}}_{(i+1)T}^\tran  \M \bm{s}_{(i+1)T} 
        \notag \\ &
        + \bm{\mathcal{A}}_{(i+1)T}^\tran\
         \sum_{t=1}^{T-1} \prod_{t'=1}^t (I - \M \bm{\mathcal{H}}_{(i+1)T-t'}  )\M \bm{s}_{(i+1)T - t}.
    \end{align}
    We can then take the Jordan decomposition of the combination matrix:
	\begin{align}
		\bm{A}_{iT+t} = \bm{V}_{iT+t}\bm{J}_{iT+t} \bm{V}_{iT+t}^{-1},
	\end{align}
	where we use the double stochastic nature of $\bm{A}_{iT+t}$ and define:
	\begin{align}
		\bm{J}_{iT+t} &\eqdef \begin{bmatrix}
			1 & 0 \\ 0 & \bm{J}_{\epsilon,iT+t}
		\end{bmatrix},
		\\
		\bm{V}_{iT+t} &\eqdef \begin{bmatrix}
			\frac{1}{K}\mathds{1} & \bm{V}_{R,iT+t}
		\end{bmatrix},
		\\
		\bm{V}_{iT+t}^{-1} & \eqdef \begin{bmatrix}
			\mathds{1}^\tran \\  \bm{V}_{L,iT+t}^\tran
		\end{bmatrix},
	\end{align}
	with $\bm{J}_{\epsilon,iT+t}$ consisting of Jordan blocks with $\epsilon$ on the lower off-diagonal, e.g.:
	\begin{align}
		\begin{bmatrix}
			\lambda & 0 & 0 \\
			\epsilon & \lambda & 0 \\
			0 & \epsilon & \lambda
		\end{bmatrix}.
	\end{align}
Then, if we multiply the error recursion by $\bm{\mathcal{V}}_{(i+1)T}^\tran \eqdef \bm{V}_{(i+1)T}^\tran \otimes I$ from the left: 
\begin{align}
	&\bm{\mathcal{V}}_{(i+1)T}^\tran  \wse_{(i+1)T} 
 \notag \\ 
    &= \bm{\mathcal{J}}_{(i+1)T}\bm{\mathcal{V}}_{(i+1)T}^\tran \prod_{t=1}^T ( I - \M \bm{\mathcal{H}}_{(i+1)T-t})(\bm{\mathcal{V}}_{iT}^{-1})^\tran \bm{\mathcal{V}}_{iT}^\tran \wse_{iT}  
	\notag \\
	&  -\bm{\mathcal{J}}_{(i+1)T}\bm{\mathcal{V}}_{(i+1)T}^\tran  \left \{ I \hspace{-1mm}+\hspace{-1mm} \sum_{t=1}^{T-1} \prod_{t'=1}^t (I \hspace{-1mm}-\hspace{-1mm} \M \bm{\mathcal{H}}_{(i+1)T-t'}  )\right\}\M b
        \notag \\
        & + \bm{\mathcal{J}}_{(i+1)T}\bm{\mathcal{V}}_{(i+1)T}^\tran \M \bm{s}_{(i+1)T}
        \notag \\
        &+\hspace{-1mm} \bm{\mathcal{J}}_{(i+1)T}\bm{\mathcal{V}}_{(i+1)T}^\tran
        \hspace{-0.5mm}\sum_{t=1}^{T-1} \hspace{-0.5mm}\prod_{t'=1}^t (I \hspace{-1mm}-\hspace{-1mm} \M \bm{\mathcal{H}}_{(i+1)T-t'}  )\M \bm{s}_{(i+1)T - t}
        ,
\end{align}	
and if we define:
	\begin{align} 
		\bar{\ws}_{i+1} &\eqdef
            \frac{1}{K}\sum_{k=1}^K \we_{k,(i+1)T}, 
		\\
		\check{\ws}_{i+1} &\eqdef 
            (\bm{V}_{R,(i+1)T}^\tran \otimes I) \wse_{(i+1)T}  
		\\
            \bar{\bm{b}}_{i+1} &\eqdef 
            \frac{1}{K}\sum_{k=1}^K \bm{b}'_{k,i+1} ,
            \\
		\check{\bm{b}}_{i+1} &\eqdef 
            (\bm{J}_{\epsilon,(i+1)T}\bm{V}_{R,(i+1)T}^\tran \otimes I) \bm{b}_{i+1}' ,
            \\
		\bar{\bm{s}}_{i+1} & \eqdef \frac{1}{K}\sum_{k=1}^k {\bm{s}}'_{k,i+1} ,
		\\
		\check{\bm{s}}_{i+1} & \eqdef 
            (\bm{J}_{\epsilon,(i+1)T}\bm{V}_{R,(i+1)T}^\tran \otimes I)\bm{s}'_{i+1} ,
	\end{align}
	where we introduce $\bm{b}'_{i+1}$ and $\bm{s}'_{i+1}$ to encapsulate all the terms with the bias and gradient noise respectively, we can write:
	\begin{align}\label{eq:transErrRec}
		\begin{bmatrix}
			\bar{\ws}_{i+1} \\
			\check {\ws}_{i+1}
		\end{bmatrix} &= \bm{\mathcal{B}}_{i+1}\begin{bmatrix}
		\bar{\ws}_{i}\\
		\check {\ws}_{i}
		\end{bmatrix}  + \begin{bmatrix}
				\bar{\bm{b}}_{i+1} + \bar{\bm{s}}_{i+1} \\
					\check{\bm{b}}_{i+1} + \check{\bm{s}}_{i+1}
		\end{bmatrix},
	\end{align}
	where we define the entries of the matrix $\bm{\mathcal{B}}_{i+1}$: 
	\begin{align}
		\bm{B}_{11,i+1} &\eqdef I - \frac{1}{K}\sum_{k=1}^K \sum_{t=1}^{T}\bm{\mu}_{k,i}\bm{H}_{k,(i+1)T-t} + O(\mu^2), \\
		\bm{B}_{12,i+1} & \eqdef 
                - \frac{1}{K}\row{\bm{\mu}_{k,i} \sum_{t=1}^T\bm{H}_{k,(i+1)T-t}} \bm{V}_{L,iT} \otimes I
                \notag \\
                &\quad + O(\mu^2),
                \\
		\bm{B}_{21,i+1} & \eqdef 
                - \bm{J}_{\epsilon,(i+1)T}\bm{V}_{R,(i+1)T}^\tran \otimes I 
                \notag \\
                &\quad \times \col{\bm{\mu}_{k,i} \sum_{t=1}^T\bm{H}_{k,(i+1)T-t}}+ O(\mu^2),
                \\
		\bm{B}_{22,i+1} & \eqdef 
            \bm{J}_{\epsilon,(i+1)T}\bm{V}_{R,(i+1)T}^\tran \bm{V}_{L,iT}\otimes I
            \notag \\ &\quad 
            - \bm{J}_{\epsilon,(i+1)T} \bm{V}_{R,(i+1)T}^\tran \otimes I \sum_{t=1}^T\M \bm{\mathcal{H}}_{(i+1)T-t}
            \notag \\ &\quad \times
            \bm{V}_{L,iT}\otimes I  + O(\mu^2),
	\end{align}
 which can be calculated by using the following identity:
 \begin{align}
     &\prod_{t=1}^T (I-\M\bm{\mathcal{H}}_{(i+1)T-t} )
     \notag \\
     &= I - \sum_{t=1}^T \M \bm{\mathcal{H}}_{(i+1)T-t}
     \notag \\ &\quad
     + \sum_{t=2}^T (-1)^t \sum_{1\leq t_1\leq t_2 \leq \cdots t_t\leq T} \prod_{j=1}^t\M\bm{\mathcal{H}}_{(i+1)T-t_j}
     \notag \\
     &= I - \sum_{t=1}^T \M \bm{\mathcal{H}}_{(i+1)T-t} +O(\mu^2).
 \end{align}
 From the bounds of the Hessian, we can conclude:
\begin{align}\label{eq:bdH}
    \nu I \leq \bm{H}_{k,iT+1} \leq \delta I,
\end{align}
and can consequently bound the the elements of the matrix as follows:
\begin{align}
    \mathbb{E} \Vert \bm{B}_{11,i+1}\Vert^2 &\leq \left(1 - \frac{T\nu \mu}{K}\sum_{k=1}^K q_k + O(\mu^2) \right)^2,
    \\
    \mathbb{E} \Vert \bm{B}_{12,i+1}\Vert^2 &\leq \frac{T^2\delta^2 \mu^2}{K^2}\mathbb{E} \Vert \bm{V}_{L,(i+1)T}\Vert^2  \sum_{k=1}^K q_k + O(\mu^4)
    \notag \\
    &= T^2 \sum_{k=1}^K q_k O(\mu^2),
    \\
    \mathbb{E} \Vert \bm{B}_{21,i+1}\Vert^2 &\leq T^2\delta^2 \mu^2 \mathbb{E}\Vert \bm{J}_{\epsilon,(i+1)T}\Vert^2 \Vert \bm{V}^\tran_{R,(i+1)T}\Vert^2 \sum_{k=1}^K q_k
    \notag \\ &\quad 
    + O(\mu^4)
    \notag \\
    &= T^2 \sum_{k=1}^K q_k O(\mu^2),
    \\
    \mathbb{E}\Vert \bm{B}_{22,i+1} -& \bm{J}_{\epsilon,(i+1)T}\bm{V}_{R,(i+1)T}^\tran \bm{V}_{L,iT}\otimes I\Vert^2 
    \notag \\ &
    \leq T^2\delta^2 \mu^2 \mathbb{E} \Vert \bm{J}_{\epsilon,(i+1)T}\Vert^2 \Vert \bm{V}^\tran_{R,(i+1)T}\Vert^2 \Vert \bm{V}_{L,iT}\Vert^2  
    \notag \\ &\quad \times
    \sum_{k=1}^K q_k + O(\mu^4)
    \notag \\ 
    &=T^2 \sum_{k=1}^K q_k O(\mu^2).
\end{align}
Next, we bound the bias terms. Thus we split the bias into two terms:
\begin{align}
    \bar{\bm{b}}_{i+1} &=  \frac{(T+O(\mu^2))}{K}\sum_{k=1}^K \bm{\mu}_{k,i} \grad{w}J_k(w^o) 
    \notag \\ & \quad 
    - \frac{1}{K}\sum_{k=1}^K\sum_{t=1}^{T-1}\sum_{t'=1}^t \bm{\mu}_{k,i}^2 \bm{H}_{k,(i+1)T-t'} \grad{w}J_k(w^o)
    \notag \\
    &= \bar{\bm{b}}_{1,i+1}+ \bar{\bm{b}}_{2,i+1},
\end{align}
and taking the conditional expectation over the past models the first term disappears following from the definition of $w^o$, and we get:
\begin{align}
    &\mathbb{E} \{\bar{\bm{b}}_{i+1}| \mathcal{F}_{(i+1)T-1}\} 
    \notag \\
    &= \mathbb{E} \{\bar{\bm{b}}_{2,i+1}| \mathcal{F}_{(i+1)T-1}\}
    \notag \\ &= 
    - \frac{\mu^2}{K}\sum_{k=1}^K\sum_{t=1}^{T-1}\sum_{t'=1}^t q_k  \bm{H}_{k,(i+1)T-t'} \grad{w}J_k(w^o). 
\end{align}
Next, we bound the second-order moment:
\begin{align}
    &\mathbb{E} \{\Vert \bar{\bm{b}}_{i+1}\Vert^2  | \mathcal{F}_{(i+1)T-1}\} 
    \notag \\
    &= \frac{1}{K^2}\sum_{k=1}^K \mathbb{E} \Bigg\{ \bigg\Vert (T+O(\mu^2)) \bm{\mu}_{k,i} \grad{w}J_k(w^0) 
    \notag \\
    &\quad  - \bm{\mu}^2_{k,i} \sum_{t=1}^{T-1} \sum_{t'=1}^t \bm{H}_{k,(i+1)T-t'}\grad{w}J_k(w^o) \bigg\Vert \Bigg| \mathcal{F}_{(i+1)T-1}  \Bigg \}  
    \notag \\
    &= \frac{1}{K^2}\sum_{k=1}^K  (T+O(\mu^2))^2\mu^2 q_k \Vert \grad{w}J_k(w^o)\Vert^2 
    \notag \\
    &\quad + \mu^4 q_{k} \left\Vert \sum_{t=1}^{T-1}\sum_{t'=1}^t \bm{H}_{k,(i+1)T-t'} \grad{w}J_k(w^o) \right\Vert^2 
    \notag \\
    &\quad - 2(T+O(\mu^2))\mu^3 q_k \left\Vert \sum_{t=1}^{T-1}\sum_{t'=1}^t \bm{H}_{k,(i+1)T-t'} \grad{w}J_k(w^o) \right\Vert 
    \notag \\
    &\leq \frac{1}{K^2}\sum_{k=1}^K  \left( (T+O(\mu^2))^2 \mu^2 + O(T^4)\mu^4\right)q_k\Vert \grad{w}J_k(w^o)\Vert^2 
    \notag \\
    \notag \\
    &\quad - \delta (T+O(\mu^2)) T(T-1)\mu^3 q_k \Vert \grad{w}J_k(w^o)\Vert 
    \notag \\
    &= T^2O(\mu^2)\sum_{k=1}^K q_{k}\Vert \grad{w}J_k(w^o)\Vert^2+ O(\mu^4), 
\end{align}
while the second term on its own corresponds to a higher-order term in $\mu$:
\begin{align}
    \mathbb{E} \{ \Vert \bar{\bm{b}}_{2,i+1}\Vert^2 | \mathcal{F}_{(i+1)T-t}\} &\leq O(\mu^4),
\end{align}
and: 
\begin{align}
    &\mathbb{E} \{\Vert \check{\bm{b}}_{i+1}\Vert^2  | \mathcal{F}_{(i+1)T-1}\} 
    \notag \\
    &\leq \mathbb{E} \Vert \bm{J}_{\epsilon,(i+1)T}\Vert^2 \Vert \bm{V}_{R,(i+1)T}\Vert^2  
    \notag \\ & \quad \times
    \mathbb{E} \Bigg\{ 2(T+O(\mu^2))^2 \sum_{k=1}^k \bm{\mu}_{k,i}^2\Vert \grad{w}J_k(w^o)\Vert^2 
    \notag \\ &\quad
    + T(T-1) \delta^2 \sum_{k=1}^K \bm{\mu}_{k,i}^4 \Vert \grad{w}J_k(w^o)\Vert^2  \Bigg\}
    \notag \\
    &= \left ( 2(T\hspace{-0.8mm}+\hspace{-0.8mm}O(\mu^2))^2 \mu^2 \hspace{-0.8mm}+\hspace{-0.8mm} T(T\hspace{-0.8mm}-\hspace{-0.8mm}1) \delta^2 \mu^4 \right)\sum_{k=1}^K q_k \Vert \grad{w}J_{k}(w^o)\Vert^2 
    \notag \\
    &= T^2O(\mu^2)\sum_{k=1}^K q_{k}\Vert \grad{w}J_k(w^o)\Vert^2 +O(\mu^4).
\end{align}
Next, to bound the gradient noise, first we have:
\begin{align}
    &\mathbb{E} \{ \Vert \bar{\bm{s}}_{i+1} \Vert^2 | \bm{\mathcal{F}}_{(i+1)T-1}\} 
    \notag \\
    &= \frac{1}{K^2}\sum_{k=1}^K \mathbb{E}\Bigg \{ \hspace{-0.8mm}\bm{\mu}_{k,i}^2 \Bigg\Vert \hspace{-0.8mm} \sum_{t=0}^{T-1}\bm{s}_{k,(i+1)T-t} \hspace{-0.8mm}+\hspace{-0.8mm} O(\mu^2)\sum_{t=1}^{T-1}\bm{s}_{k,(i+1)T-t}
    \notag \\
    &\quad - \sum_{t=1}^{T-1}\sum_{t'=1}^t \bm{H}_{k,(i+1)T-t'}\bm{s}_{k,(i+1)T-t} \Bigg\Vert^2 \Bigg | \bm{\mathcal{F}}_{(i+1)T-1}\Bigg\}
    \notag \\
    &\leq \frac{3\mu^2}{K^2}\sum_{k=1}^K q_k \Bigg\{ \sum_{t=0}^{T-1}\mathbb{E}\{\Vert \bm{s}_{k,(i+1)T-t}\Vert^2 | \bm{\mathcal{F}}_{(i+1)T-1}\}
    \notag \\ &\quad
    + O(\mu^4)\sum_{t=1}^{T-1}\mathbb{E}\{\Vert \bm{s}_{k,(i+1)T-t}\Vert^2 | \bm{\mathcal{F}}_{(i+1)T-1}\} 
    \notag \\ &\quad 
    + \delta^2\sum_{t=1}^{T-1}t \mathbb{E} \{\Vert  \bm{s}_{k,(i+1)T-t}\Vert^2  | \bm{\mathcal{F}}_{(i+1)T-1} \} \Bigg\} 
    \notag \\ 
    &\leq \frac{3\mu^2}{K^2}\sum_{k=1}^K q_k \bigg\{\beta_s^2 \Vert \we_{k,(i+1)T-1}\Vert^2 + \sigma_s^2
    \notag \\
    &\quad + (T\delta^2 +O(\mu^4))\sum_{t=1}^{T-1} (\beta_s^2\Vert \we_{k,(i+1)T-t-1}\Vert^2 + \sigma_s^2) \bigg\}
    \notag \\
    &= O(\mu^2) \sum_{k=1}^K \sum_{t=0}^{T-1}\Vert \we_{k,(i+1)T-t-1}\Vert^2 + O(\mu^2)\sigma_s^2 + O(\mu^4),
\end{align}
and following similar steps we can show: 
\begin{align}
    &\mathbb{E} \{ \Vert \check{\bm{s}}_{i+1} \Vert^2 | \bm{\mathcal{F}}_{(i+1)T-1}\} 
    \notag \\
    &\leq O(\mu^2)\sum_{k=1}^K \mathbb{E} \{ \Vert \bm{s}_{k,(i+1)T}\Vert^2 |\bm{\mathcal{F}}_{(i+1)T-1} \} 
    \notag \\
    &\quad + (O(\mu^2) + O(\mu^4)) \sum_{k=1}^K\sum_{t=1}^{T-1} \mathbb{E} \{ \Vert \bm{s}_{k,(i+1)T-t}\Vert^2 |\bm{\mathcal{F}}_{(i+1)T-1} \} 
    \notag \\ 
    &\leq O(\mu^2)\sum_{k=1}^K \sum_{t=0}^{T-1} \Vert \we_{k,(i+1)T-t-1}\Vert^2 + O(\mu^2)\sigma_s^2 + O(\mu^4).
\end{align}
Then taking the expectation again:
\begin{align}
    \mathbb{E}  \Vert \bar{\bm{s}}_{i+1} \Vert^2 
    &\leq O(\mu^2) \sum_{k=1}^K \sum_{t=1}^{T-1} (1\hspace{-0.8mm}-\hspace{-0.8mm}O(\mu)\hspace{-0.8mm}+\hspace{-0.8mm}O(\mu^2))^{T-t} \mathbb{E}\Vert \we_{k,iT}\Vert^2
    \notag \\ &\quad
    + O(\mu^2)\sigma_s^2 + O(\mu^4) 
    \notag \\ 
    &\leq O(\mu^2) \mathbb{E}\Vert \wse_{iT}\Vert^2 + O(\mu^2)\sigma_s^2 + O(\mu^4),
    \\
    \mathbb{E}\Vert\check{\bm{s}}_{i+1}\Vert^2 &\leq O(\mu^2) \mathbb{E}\Vert \wse_{iT}\Vert^2 + O(\mu^2)\sigma_s^2 + O(\mu^4).
\end{align}

Then, we take the conditional mean of the $\ell_2-$norm of the individual equations over all the past models $\mathcal{F}_{(i+1)T-1}$. We use the independence of the models from the bias and gradient noise and since the gradient noise is zero-mean in addition to applying Jensen's inequality for some positive constants $\bar{\tau}, \check{\tau} < 1$:
\begin{align}
	&\ \mathbb{E} \{ \Vert \bar{\ws}_{i+1}\Vert^2 | \mathcal{F}_{(i+1)T-1}\} \\
    &\ \leq \frac{\mathbb{E}\Vert \bm{B}_{11,(i+1)T}\Vert^2}{\bar{\tau}} \Vert \bar{\ws}_{i}\Vert^2 
    \notag + \frac{2\mathbb{E}\Vert\bm{B}_{12,i+1}\Vert^2}{1-\bar{\tau}} \Vert \check{\ws}_{i}\Vert^2 
	\notag \\
        &\quad + \frac{2}{1-\bar{\tau}}\mathbb{E}\{ \Vert \bar{\bm{b}}_{2,i+1}\Vert^2 | \mathcal{F}_{(i+1)T-1}\} 
        \notag \\
        &\quad + \mathbb{E} \{\Vert \bar{\bm{b}}_{1,i+1}\Vert^2  | \mathcal{F}_{(i+1)T-1}\} + \mathbb{E} \{\Vert \bar{\bm{s}}_{i+1}\Vert^2  | \mathcal{F}_{(i+1)T-1}\}, \nonumber 
\end{align}	
and:
\begin{align}
	&\mathbb{E} \{ \Vert \check{\ws}_{i+1}\Vert^2 | \mathcal{F}_{(i+1)T-1}\} 
 \notag \\
 &\leq \frac{\mathbb{E}\Vert \bm{J}_{\epsilon,(i+1)T}\bm{V}_{R,(i+1)T}^\tran \bm{V}_{L,iT} \Vert^2}{\check{\tau}} \Vert \check{\ws}_{i}\Vert^2  + \frac{3}{1-\check{\tau}}
	\notag \\
	&\quad \times \Big(\mathbb{E} \Vert \bm{B}_{22,i+1} - \bm{J}_{\epsilon,(i+1)T}\bm{V}_{R,(i+1)T}^\tran \bm{V}_{L,iT}\otimes I\Vert^2 \Vert \check{\ws}_{i}\Vert^2 
 \notag \\
	&\quad + \mathbb{E}\Vert \bm{B}_{21,i+1} \Vert^2 \Vert \bar{\ws}_{i}\Vert^2  + \mathbb{E}\{\Vert \check{\bm{b}}_{i+1}\Vert^2 | \mathcal{F}_{(i+1)T-1}\} \Big)
    \notag \\ &	 \quad 
    + \mathbb{E} \{\Vert \check{\bm{s}}_{i+1}\Vert^2  | \mathcal{F}_{(i+1)T-1}\}.
\end{align}

Next setting:
\begin{align}
    \bar{\tau} &= \sqrt{\mathbb{E}\Vert \bm{B}_{11,i+1}\Vert^2} = 1 - \frac{T\nu \mu}{K}\sum_{k=1}^K q_k + O(\mu^2) ,
    \\
    \check{\tau} &= \sqrt{\mathbb{E}\Vert \bm{J}_{\epsilon,(i+1)T}\bm{V}_{R,(i+1)T}^\tran \bm{V}_{L,iT} \Vert^2} = O(1),
\end{align}
we get by retaking the expectation:
\begin{align}
    \mathbb{E} \Vert \bar{\ws}_{i+1}\Vert^2 &\leq 
\bar{\tau} \mathbb{E} \Vert \bar{\ws}_{i}\Vert^2  
     + O(\mu) \mathbb{E} \Vert \check{\ws}_{i}\Vert^2  
     \notag \\ 
     & \quad + O(\mu^2) + O(\mu^2) \left ( \beta_s^2 \mathbb{E} \Vert \wse_{iT}\Vert^ 2 + \sigma_s^2 \right) 
     \notag \\
     &\leq \bar{\gamma} \mathbb{E} \Vert \bar{\ws}_{i}\Vert^2  + O(\mu) \mathbb{E} \Vert \check{\ws}_{i}\Vert^2 
     + O(\mu^2 ), 
\end{align}
 and:
 \begin{align}
     \mathbb{E} \Vert \check{\ws}_{i+1}\Vert^2 &\leq \check{\tau} \mathbb{E} \Vert \check{\ws}_i\Vert^2 
     \notag \\
     &\quad + O(\mu^2) \mathbb{E} \Vert \check{\ws}_i\Vert^2   +T^2 \sum_{k=1}^K q_k O(\mu^2)\mathbb{E} \Vert \bar{\ws}_i\Vert^2 
     \notag \\
     &\quad  + O(\mu^2)T^2\sum_{k=1}^K q_k \Vert \grad{w}J_k(w^o)\Vert^2 + O(\mu^2) \sigma_s^2 
     \notag \\ &\quad 
     + O(\mu^4)
     \notag \\ 
     &\leq \check{\gamma} \mathbb{E} \Vert \check{\ws}_{i-1}\Vert^2 + O(\mu^2)\mathbb{E} \Vert \bar{\ws}_i\Vert^2 + O(\mu^2),
 \end{align}
 with:
 \begin{align}
     \bar{\gamma} &\eqdef \bar{\tau} + O(\mu^2)\beta_s^2, \\
     \check{\gamma} &= \check{\tau}+ T^2\sum_{k=1}^k q_k \Vert \grad{w}J_k(w^o)\Vert^2 O(\mu^2 ) + O(\mu^2)\beta_s^2 .
 \end{align}
Defining the matrix $\Gamma$:
\begin{align}
    \Gamma \eqdef \begin{bmatrix}
        \bar{\gamma} & O(\mu) \\
        O(\mu^2) & \check{\gamma}
    \end{bmatrix},
\end{align}
then:
\begin{align}
    \begin{bmatrix}
        \mathbb{E} \Vert \bar{\ws}_{i+1}\Vert^2 \\
        \mathbb{E} \Vert \check{\ws}_{i+1}\Vert^2
    \end{bmatrix} &\leq \Gamma \begin{bmatrix}
        \mathbb{E} \Vert \bar{\ws}_{i}\Vert^2 \\
        \mathbb{E} \Vert \check{\ws}_{i}\Vert^2
    \end{bmatrix} + O(\mu^2)\mathds{1}
    \notag \\ 
    &\leq \Gamma^i \begin{bmatrix}
        \mathbb{E} \Vert \bar{\ws}_{0}\Vert^2 \\
        \mathbb{E} \Vert \check{\ws}_{0}\Vert^2
    \end{bmatrix} + (I-\Gamma)^{-1}(I-\Gamma^i)O(\mu^2)\mathds{1}.
\end{align}
Taking the limit as $i$ goes to infinity we get the desired result:
\begin{align}
    \limsup_{i\to\infty} \begin{bmatrix}
        \mathbb{E} \Vert \bar{\ws}_{i+1}\Vert^2 \\
        \mathbb{E} \Vert \check{\ws}_{i+1}\Vert^2
    \end{bmatrix} &\leq (I-\Gamma) O(\mu^2) \mathds{1}
    \notag \\
    &= \begin{bmatrix}
        1/O(\mu) & O(1) \\
        O(\mu) & O(1)
    \end{bmatrix} O(\mu^2)\mathds{1}
\end{align}
and:
\begin{align}
    \limsup_{i\to\infty} \mathbb{E} \Vert \wse_{(i+1)T}\Vert^2 &\leq \limsup_{i\to\infty} \mathbb{E} \left\Vert \left(\bm{\mathcal{V}}_{(i+1)T }^{-1} \right)^\tran \begin{bmatrix}
        \bar{\ws}_{i+1} \\ \check{\ws}_{i+1}  
    \end{bmatrix}
    \right\Vert^2
    \notag \\
    &\leq \limsup_{i\to\infty} \Vert \bm{\mathcal{V}}_{(i+1)T }^{-1} \Vert^2 \left( \mathbb{E} \Vert \bar{\ws}_{i+1}\Vert^2 \right.
    \notag \\ &\quad \left. 
    + \mathbb{E} \Vert \check{\ws}_{i+1}\Vert^2 \right)
    \notag \\
    &= O(\mu).
\end{align}

\section{Proof of Theorem \ref{thrm:4thStab}}\label{app:thrm4thStab}
The proof follows closely the proof of Theorem 9.2 in \cite{sayed2014adaptation}. Therefore, we present a more concise version with the key differences. We start by using the following relation:
\begin{align}
    \Vert x + y \Vert ^4 &\leq \Vert x\Vert^4 + 3\Vert y\Vert^4 + 8\Vert x\Vert^2\Vert y\Vert^2 + 4\Vert x\Vert^2 a^\tran b, 
\end{align}
by first setting:
\begin{align}
    x &= \bm{B}_{11,i+1}\bar{\ws}_{i} + \bm{B}_{12,i+1}\check{\ws}_i + \bar{\bm{b}}_{2,i+1} , \\
    y &= \bar{\bm{b}}_{1,i+1} + \bar{\bm{s}}_{i+1},
\end{align}
to get:
\begin{align}
    &\ \mathbb{E} \Vert \bar{\ws}_{i+1}\Vert^4 \nonumber \\
    &\leq \mathbb{E}\Vert \bm{B}_{11,i+1}\Vert \mathbb{E}\Vert \bar{\ws}_i\Vert^4 + \frac{2\mathbb{E} \Vert \bm{B}_{12,i+1}\Vert^4 }{(1-\mathbb{E}\Vert \bm{B}_{11,i+1}\Vert)^3} \mathbb{E} \Vert \check{\ws}_i\Vert^4
    \notag \\ &\quad 
    + \frac{2}{(1-\mathbb{E}\Vert \bm{B}_{11,i+1}\Vert)^3} \mathbb{E} \Vert \bar{\bm{b}}_{2,i+1}\Vert^4 
    \notag \\
    &\quad 
    + 3\mathbb{E}\Vert \bar{\bm{b}}_{1,i+1}\Vert^4 + 3\mathbb{E}\Vert \bar{\bm{s}}_{i+1}\Vert^4 
    \notag \\
    &\quad 
    +\Big((1-O(\mu)+O(\mu^2))\mathbb{E}\Vert \bar{\ws}_i\Vert^2 + O(\mu^2) \mathbb{E}\Vert \check{\ws}_i\Vert^2 
    \notag \\ &\quad 
    + O(\mu^3)\Big)(\mathbb{E}\Vert \bar{\bm{b}}_{1,i+1}\Vert^2 + \mathbb{E}\Vert \bar{\bm{s}}_{i+1}\Vert^2 )
    \notag \\
    &\leq (1-O(\mu)+O(\mu^2))\mathbb{E} \Vert \bar{\ws}_i\Vert^4 + O(\mu) \mathbb{E} \Vert \check{\ws}_i\Vert^4 
    \notag \\
    &\quad + O(\mu^4)  + \Big((1-O(\mu)+O(\mu^2))\mathbb{E}\Vert \bar{\ws}_i\Vert^2 
    \notag \\ &\quad 
    + O(\mu^2) \mathbb{E}\Vert \check{\ws}_i\Vert^2  + O(\mu^3)\Big)
    \Big( O(\mu^2)\mathbb{E}\Vert \bar{\ws}_i\Vert^2
    \notag \\ &\quad 
    + O(\mu^2)\mathbb{E}\Vert \check{\ws}_i\Vert^2  + O(\mu^2) \Big)
    \notag \\
    &\leq (1-O(\mu)+O(\mu^2))\mathbb{E} \Vert \bar{\ws}_i\Vert^4 + O(\mu) \mathbb{E} \Vert \check{\ws}_i\Vert^4  
    \notag \\ &\quad 
    + O(\mu^2)\mathbb{E} \Vert \bar{\ws}_i\Vert^3 + O(\mu^3)\mathbb{E} \Vert \check{\ws}_i\Vert^2 + O(\mu^4)
\end{align}
and then to:
\begin{align}
    x &= \bm{B}_{22,i+1}\check{\ws}_{i} + \bm{B}_{21,i+1}\bar{\ws}_i + \check{\bm{b}}_{i+1} , \\
    y &=  \check{\bm{s}}_{i+1},
\end{align}
to get:
\begin{align}
    \mathbb{E}\Vert \check{\ws}_{i+1}\Vert^4 &\leq (\rho(J_{\epsilon}) +O(\mu^2)) \mathbb{E} \Vert \check{\ws}_i\Vert^4 + O(\mu) \mathbb{E} \Vert \bar{\ws}_i\Vert^4
    \notag \\ 
    &\quad + O(\mu^2) \mathbb{E} \Vert \check{\ws}_i\Vert^2 + O(\mu^4)\mathbb{E} \Vert \bar{\ws}_i\Vert^2 + O(\mu^4).
\end{align}
The desired result easily follows:
\begin{align}
    \limsup_{i\to\infty} \mathbb{E}\Vert \wse_{(i+1))T}\Vert^4&= \limsup_{i\to \infty} \mathbb{E}\Vert \bar{\ws}_{i+1}\Vert^4 + \mathbb{E}\Vert \bar{\ws}_{i+1}\Vert^4 
    \notag \\
    &= O(\mu^2).
\end{align}

\section{Proof of Theorem \ref{thrm:approxErr}} \label{app:thrmApproxErr}
    Introducing the variable:
    \begin{align}
        \bm{z}_{iT+t} \eqdef \wse_{iT+t} - \wse'_{iT+t},
    \end{align}
    and subtracting the long-term model \eqref{eq:longTermRec} from the error recursion \eqref{eq:errRec}, we have:
    \begin{align}
        \bm{z}_{iT+t} = \A^\tran (I - \M \mathcal{H}) \bm{z}_{iT+t-1} + \A^\tran \M \bm{c}_{iT+t-1}.
    \end{align}
    Then, we can write at $t=T$:
    \begin{align}
        \bm{z}_{(i+1)T} &= \bm{\mathcal{A}}_{(i+1)T}^\tran (I - \M \mathcal{H} )^T \bm{z}_{iT} 
        \notag \\ &\quad
        + \bm{\mathcal{A}}_{(i+1)T}^\tran\M\sum_{t=1}^{T-1} (I-\M \mathcal{H})^t\bm{c}_{(i+1)T-t}.
    \end{align}
    Then if we multiply from the left by $\bm{\mathcal{V}}_{(i+1)T}^\tran$, we transform the recursion to:
    \begin{align}
        \begin{bmatrix}
            \bm{\bar{z}}_{i+1} \\
            \bm{\check{z}}_{i+1}
        \end{bmatrix}
        &= \bm{B}'_{i+1} \begin{bmatrix}
            \bm{\bar{z}}_{i} \\
            \bm{\check{z}}_{i}
        \end{bmatrix}
        + \begin{bmatrix}
            \bar{\bm{c}}_{i} \\
            \check{\bm{c}}_i
        \end{bmatrix}.
    \end{align}
Similar to in the proof of Theorem \ref{thrm:2ndStab}, we can show:
\begin{align}
    \mathbb{E}\Vert\bar{\bm{z}}_{i+1} \Vert^2 &\leq (1-O(\mu) + O(\mu^2)) \mathbb{E}\bar{\bm{z}}_{i} + O(\mu) \mathbb{E}\check{\bm{z}}_{i+1} 
    \notag \\
    &\quad + O(\mu^{-1})\Vert \bar{\bm{c}}_{i-1}\Vert^2,
    \\
    \mathbb{E}\Vert\check{\bm{z}}_{i+1} \Vert^2 &\leq (\rho(J_{\epsilon}) +  O(\mu^2)) \mathbb{E}\Vert\check{\bm{z}}_{i}\Vert^2+ O(\mu^2) \mathbb{E}\Vert\bar{\bm{z}}_{i} \Vert^2
    \notag \\
    &\quad + O(1)\Vert \check{\bm{c}}_{i-1}\Vert^2.
\end{align}
We use the following bounds:
\begin{align}
    &\mathbb{E} \left\Vert \bm{\mathcal{A}}_{(i+1)T}^\tran \M \sum_{t=1}^{T-1} (I-\M \mathcal{H})^t\bm{c}_{(i+1)T-t} \right\Vert^2
    \notag \\
    &\leq O(\mu^2)\sum_{t=1}^{T-1} (1-O(\mu))^{2t}\mathbb{E}\Vert\bm{c}_{(i+1))T-t}\Vert^2 
    \notag \\
    &\leq O(\mu^2 )\sum_{t=1}^{T-1}(1-O(\mu))^{2t} \mathbb{E}\Vert \wse_{(i+1)T-t}\Vert^4
\end{align}
 Using the forth-order stability of the algorithm (Theorem \ref{thrm:4thStab}), we get the desired result:
 \begin{align}
     \limsup_{i\to \infty} \mathbb{E}\Vert \bm{z}_i\Vert^2 &= \mathds{1}^\tran \begin{bmatrix}
         O(\mu) & -O(\mu)\\
         -O(\mu^2) & O(1)
     \end{bmatrix}^{-1}\begin{bmatrix}
         O(\mu) \\ O(\mu^2)
     \end{bmatrix} 
     \notag \\ & \quad \times
     \limsup_{i\to \infty} \sum_{t=1}^T (1-O(\mu))^t\mathbb{E}\Vert \wse_{(i+1)T-t)}\Vert^4 
     \notag \\
     &= O(\mu^2).
 \end{align}
 Finally, since:
 \begin{align}
     &\ \mathbb{E}\Vert \wse'_{(i+1)T}\Vert^2 \nonumber \\
     &\leq  \mathbb{E}\Vert \wse'_{(i+1)T} - \wse_{(i+1)T}\Vert^2 + \mathbb{E}\Vert \wse_{(i+1)T}\Vert^2 
     \notag \\
     &\quad + 
     2\sqrt{\mathbb{E}\Vert\wse'_{(i+1)T} - \wse_{(i+1)T}\Vert^2  \mathbb{E}\Vert \wse_{(i+1)T} \Vert^2 },
 \end{align}
 we get: 
 \begin{align}
     \limsup_{i\to \infty} \mathbb{E}\Vert \wse'_{(i+1)T}\Vert^2 - \mathbb{E}\Vert \wse_{(i+1)T}\Vert^2 &\leq O(\mu^2) + O(\mu^{3/2}).
 \end{align}

\section{Proof of Theorem \ref{thrm:stabLongTerm}}\label{app:thrmStabLongTerm}
We start by transforming the recursion using $\bm{\mathcal{V}}_{(i+1)T}^\tran$:
\begin{align}
    \wse'_{(i+1)T} &= \bm{\mathcal{A}}_{(i+1)T}^\tran (I-\M \mathcal{H})^T \wse'_{iT} 
    \notag \\ &\quad 
    + \bm{\mathcal{A}}_{(i+1)T}^\tran\sum_{t=0}^{T-1}(I-\M \mathcal{H})^t \M (b +\bm{s}_{(i+1)T-t})
\end{align}
to get the following recursion:
\begin{align}
    \begin{bmatrix}
        \bar{\ws}'_{i+1} \\
        \check{\ws}'_{i+1}
    \end{bmatrix}
    &= \bm{\mathcal{B}}'_{i+1} \begin{bmatrix}
        \bar{\ws}'_{i} \\
        \check{\ws}'_{i}
    \end{bmatrix}
    + \begin{bmatrix}
        \bar{\bm{b}}_{i+1} + \bar{\bm{s}}_{i+1}\\
        \check{\bm{b}}_{i+1}+ \check{\bm{s}}_{i+1}
    \end{bmatrix}.
\end{align}
Looking at the entries of the matrix $\bm{\mathcal{B}}'_{i+1}$:
\begin{align}
		\bm{B}'_{11,i+1} &\eqdef I - \frac{T}{K}\sum_{k=1}^K \bm{\mu}_{k,i}H_{k} + O(\mu^2), \\
		\bm{B}'_{12,i+1} & \eqdef 
                - \frac{T}{K}\row{\bm{\mu}_{k,i} {H}_{k} }\bm{V}_{L,iT} \otimes I+ O(\mu^2),
                 \\
		\bm{B}'_{21,i+1} & \eqdef 
                - T\bm{J}_{\epsilon,(i+1)T}\bm{V}_{R,(i+1)T}^\tran \otimes I \col{\bm{\mu}_{k,i} {H}_{k}}
                \notag \\ &\quad
                + O(\mu^2),
                \\
		\bm{B}'_{22,i+1} & \eqdef 
            \bm{J}_{\epsilon,(i+1)T}\bm{V}_{R,(i+1)T}^\tran \bm{V}_{L,iT}\otimes I
            \notag \\ &\quad 
            - T\bm{J}_{\epsilon,(i+1)T} \bm{V}_{R,(i+1)T}^\tran \otimes I \M {\mathcal{H}} \bm{V}_{L,iT}\otimes I
            \notag \\ &\quad 
              + O(\mu^2).
	\end{align}
Then, if we take the conditional expectation over the past models and apply Jensen's inequality with the same constants $\bar{\tau}$ and $\check{\tau}$, and then retake the expectation we can show:
\begin{align}
	&\ \mathbb{E}  \Vert \bar{\ws}'_{i+1}\Vert^2  \nonumber \\
    &\leq \mathbb{E}\Vert \bm{B}'_{11,(i+1)T}\Vert \Vert \mathbb{E}\bar{\ws}'_{i}\Vert^2  + \frac{2\mathbb{E}\Vert\bm{B}'_{12,i+1}\Vert^2}{1-\bar{\tau}} \Vert \check{\ws}'_{i}\Vert^2  \nonumber \\
    &+ \frac{2}{1-\bar{\tau}}\mathbb{E} \Vert \bar{\bm{b}}_{2,i+1}\Vert^2  + \mathbb{E} \Vert \bar{\bm{b}}_{1,i+1}\Vert^2 + \mathbb{E} \Vert \bar{\bm{s}}_{i+1}\Vert^2,
\end{align}	
and:
\begin{align}
	&\mathbb{E} \Vert \check{\ws}'_{i+1}\Vert^2  
 \notag \\
 &\leq \mathbb{E}\Vert \bm{J}_{\epsilon,(i+1)T}\bm{V}_{R,(i+1)T}^\tran \bm{V}_{L,iT} \Vert \mathbb{E}\Vert \check{\ws}'_{i}\Vert^2  + \frac{3}{1-\check{\tau}}
	\notag \\
	&\quad \times \Big(\mathbb{E} \Vert \bm{B}'_{22,i+1} - \bm{J}_{\epsilon,(i+1)T}\bm{V}_{R,(i+1)T}^\tran \bm{V}_{L,iT}\otimes I\Vert^2 \Vert \mathbb{E}\check{\ws}'_{i}\Vert^2 
 \notag \\
	&\quad + \mathbb{E}\Vert \bm{B}'_{21,i+1} \Vert^2\mathbb{E} \Vert \bar{\ws}'_{i}\Vert^2  + \mathbb{E}\Vert \check{\bm{b}}_{i+1}\Vert^2  \Big)
    \notag \\ &	 \quad 
    + \mathbb{E} \Vert \check{\bm{s}}_{i+1}\Vert^2  .
\end{align}

We can show that we get similar bounds on the bias and gradient noise thus:
\begin{align}
    \mathbb{E} \Vert \bar{\ws}'_{i+1}\Vert^2 
     &\leq \bar{\gamma}' \mathbb{E} \Vert \bar{\ws}'_{i}\Vert^2  + O(\mu) \mathbb{E} \Vert \check{\ws}'_{i}\Vert^2 
     + O(\mu^2 ), 
\\
     \mathbb{E} \Vert \check{\ws}'_{i+1}\Vert^2 &\leq \check{\gamma}' \mathbb{E} \Vert \check{\ws}'_{i-1}\Vert^2 + O(\mu^2)\mathbb{E} \Vert \bar{\ws}'_i\Vert^2 + O(\mu^2).
 \end{align}
Recursively repeating the bound and taking the limit as $i$ goes to infinity, we get the bound from Theorem \ref{thrm:stabLongTerm}.

\section{Proof of Theorem \ref{thrm:MSD}}\label{app:thrmMSD}
We start with the long-term model:
\begin{align}
    \wse'_{iT+t} &= \A^\tran (I-\M\mathcal{H})\wse'_{iT+t-1} - \A^\tran \M b 
    \notag \\ &\quad
    + \A^\tran \M \bm{s}_{iT+t},
\end{align}
and show that:
\begin{align}
    \mathbb{E}\wse'_{iT+t} &=  \mathbb{E}\A^\tran (I-\M\mathcal{H}) \wse'_{iT+t-1} - \mathbb{E} \A^\tran \M b,
\end{align}
and particularly for $t=T$:
\begin{align}
    \mathbb{E}\wse'_{(i+1)T} &= \bar{\mathcal{B}}_T\mathbb{E}\wse'_{iT} 
     -\mu\sum_{t=0}^{T-1} \bar{\mathcal{B}}_t b,
\end{align}
where we introduce the notation:
\begin{align}
    \bar{\mathcal{B}}_t &\eqdef  \mathbb{E}\bm{\mathcal{A}}_{(i+1)T}^\tran (I-\M \mathcal{H})^t
    \notag \\
    &= \Abar^\tran (I-\mu \mathcal{H})^t + \sum_{t'=1}^t \binom{t}{t'} \mu^{t'} \diag{(1-q_k)H_k^{t'}}
    .
\end{align}
Then, we can show that the average long-term model converges to:
\begin{align}
    \lim_{i\to \infty} \mathbb{E} \wse'_{(i+1)T} 
    &= -\mu (I-\bar{\mathcal{B}}_T)^{-1}  \sum_{t=0}^{T-1}\bar{\mathcal{B}}_t b = O(\mu),
\end{align}
since $I-\bar{\mathcal{B}}_T = O(1)$. We next calculate:
\begin{align}
   & \mathbb{E} \{ \wse'_{(i+1)T} \wse'^\tran_{(i+1)T} | \mathcal{F}_{(i+1)T-1} \} 
    \notag \\
    &= \mathbb{E} \{ \bm{\mathcal{A}}_{(i+1)T}^\tran (I-\M \mathcal{H})^T \wse'_{iT} \wse'^\tran _{iT} \left((I-\M \mathcal{H})^T\right)^\tran 
    \notag \\ &\quad \times 
     \bm{\mathcal{A}}_{(i+1)T} |\mathcal{F}_{(i+1)T-1}\} + \mathbb{E}\bm{\mathcal{A}}_{(i+1)T}^\tran  \hspace{-0.8mm}\sum_{t=0}^{T-1}(I\hspace{-0.8mm}-\hspace{-0.8mm}\M \mathcal{H})^t\M b
    \notag \\
    &\quad \times b^\tran \M \sum_{t=0}^{T-1} \left((I-\M \mathcal{H})^t\right)^\tran\bm{\mathcal{A}}_{(i+1)T}
    \notag \\ 
    &\quad + \mathbb{E} \Bigg\{ \bm{\mathcal{A}}_{(i+1)T}^\tran \hspace{-0.8mm}\sum_{t=0}^{T-1}(I \hspace{-0.8mm}-\hspace{-0.8mm} \M \mathcal{H})^t\M \bm{s}_{(i+1)T-t}\bm{s}_{(i+1)T-t}^\tran \M 
    \notag \\ &\quad \times 
    \sum_{t=0}^{T-1}\left((I-\M \mathcal{H})^t\right)^\tran\bm{\mathcal{A}}_{(i+1)T} \Bigg| \mathcal{F}_{(i+1)T-1} \Bigg\}
    \notag \\
    &\quad -2 \mathbb{E} \Bigg\{ \bm{\mathcal{A}}_{(i+1)T}^\tran (I-\M \mathcal{H})^T \wse'_{iT} b^\tran \M 
    \notag \\ &\quad \times \sum_{t=0}^{T-1}\left((I-\M \mathcal{H})^t\right)^\tran
    \bm{\mathcal{A}}_{(i+1)T} \Bigg| \mathcal{F}_{(i+1)T-1}  \Bigg\},
\end{align}
and arrive at:
\begin{align}
    &\mathbb{E} \{\wse'_{(i+1)T} \otimes_b \wse'_{(i+1)T} | \mathcal{F}_{(i+1)T-1}\} 
    \notag \\
    &=  \mathbb{E} \{ \bm{\mathcal{A}}_{(i+1)T}^\tran (I-\M \mathcal{H})^T \otimes_b \bm{\mathcal{A}}_{(i+1)T}^\tran (I-\M \mathcal{H})^T
    \notag \\ &\quad \times 
    \wse'_{iT} \otimes_b \wse'_{iT} 
    \notag \\ &\quad 
    + \mathbb{E} \bm{\mathcal{A}}_{(i+1)T}^\tran \sum_{t=0}^{T-1}(I-\M \mathcal{H})^t\M \otimes_b \bm{\mathcal{A}}_{(i+1)T}^\tran \
    \notag \\ &\quad \times
    \sum_{t=0}^{T-1}(I-\M \mathcal{H})^t \M \big(b\otimes_b b 
    \notag \\ &\quad 
    + \mathbb{E} \{ \bm{s}_{(i+1)T-t} \otimes_b \bm{s}_{(i+1)T-t} | \mathcal{F}_{(i+1)T-1}\}  \big)
    \notag \\ &\quad 
    -2 \mathbb{E} \bm{\mathcal{A}}_{(i+1)T}^\tran (I-\M \mathcal{H})^T \otimes_b 
    \bm{\mathcal{A}}_{(i+1)T}^\tran  
    \notag \\ &\quad \times \sum_{t=0}^{T-1}(I-\M \mathcal{H})^t\M 
    \wse'_{iT} \otimes_b b
    . 
\end{align}
Then, taking the expectation again:
\begin{align}
   & \mathbb{E} \wse'_{(i+1)T} \otimes_b \wse'_{(i+1)T}
    \notag \\
    &=  \mathbb{E} \bm{\mathcal{A}}_{(i+1)T}^\tran (I-\M \mathcal{H})^T \otimes_b \bm{\mathcal{A}}_{(i+1)T}^\tran (I-\M \mathcal{H})^T
    \notag \\ &\quad \times 
    \mathbb{E}\wse'_{iT} \otimes_b \wse'_{iT} 
    \notag \\ &\quad 
    + \mathbb{E} \bm{\mathcal{A}}_{(i+1)T}^\tran \sum_{t=0}^{T-1}(I-\M \mathcal{H})^t\M \otimes_b \bm{\mathcal{A}}_{(i+1)T}^\tran \
    \notag \\ &\quad \times
    \sum_{t=0}^{T-1}(I \hspace{-0.8mm}-\hspace{-0.8mm} \M \mathcal{H})^t \M \big(b\otimes_b\hspace{-0.5mm} b \hspace{-0.8mm}+\hspace{-0.8mm} \mathbb{E} \bm{s}_{(i+1)T-t} \otimes_b \bm{s}_{(i+1)T-t} \big)
    \notag \\ &\quad 
    -2 \mathbb{E} \bm{\mathcal{A}}_{(i+1)T}^\tran (I-\M \mathcal{H})^T \otimes_b 
    \bm{\mathcal{A}}_{(i+1)T}^\tran  
    \notag \\ &\quad \times
    \sum_{t=0}^{T-1}(I-\M \mathcal{H})^t\M \mathbb{E}\wse'_{iT} \otimes_b b. 
\end{align}
Thus, if we define:
\begin{align}
    z_{i+1} &\eqdef \mathbb{E} \wse'_{(i+1)T} \otimes_b \wse'_{(i+1)T}, 
    \\
    \mathcal{G} &\eqdef \mathbb{E} \bm{\mathcal{A}}_{(i+1)T}^\tran (I\hspace{-0.5mm}-\hspace{-0.5mm}\M \mathcal{H})^T\otimes_b \bm{\mathcal{A}}_{(i+1)T}^\tran (I\hspace{-0.5mm}-\hspace{-0.5mm}\M \mathcal{H})^T, 
    \\
    y_{i+1} &\eqdef \mathbb{E} \bm{\mathcal{A}}_{(i+1)T}^\tran \sum_{t=0}^{T-1}(I-\M \mathcal{H})^t\M \otimes_b \bm{\mathcal{A}}_{(i+1)T}^\tran \
    \notag \\ &\quad \times
    \sum_{t=0}^{T-1}(I-\M \mathcal{H})^t \M \big(b\otimes_b b
    \notag \\ &\quad 
    +\hspace{-0.5mm} \mathbb{E} \bm{s}_{(i+1)T-t}\hspace{-1mm} \otimes_b \hspace{-1mm}\bm{s}_{(i+1)T-t} \big) \hspace{-0.8mm}-\hspace{-0.8mm} 2 \mathbb{E} \bm{\mathcal{A}}_{(i+1)T}^\tran (\hspace{-0.2mm}I\hspace{-0.8mm}- \hspace{-0.8mm}\M \mathcal{H}\hspace{-0.2mm})^T
    \notag \\ &\quad 
     \otimes_b 
    \bm{\mathcal{A}}_{(i+1)T}^\tran \sum_{t=0}^{T-1}(I-\M \mathcal{H})^t\M 
    \mathbb{E}\wse'_{iT} \otimes_b b
    , 
\end{align}
and take the limit:
\begin{align}
   z_{\infty} &\eqdef \lim_{i\to \infty} z_{i+1} 
   =  \left(I-\mathcal{G} \right)^{-1} \lim_{i\to \infty}  y_{i+1} .
\end{align}
We can show that:
\begin{align}
    y_{\infty} &\eqdef \lim_{i\to \infty} y_{i+1} 
    \notag \\
    &= \mathbb{E} \bm{\mathcal{A}}_{(i+1)T}^\tran \sum_{t=0}^{T-1}(I-\M \mathcal{H})^t\M \otimes_b \bm{\mathcal{A}}_{(i+1)T}^\tran \
    \notag \\ &\quad \times
    \sum_{t=0}^{T-1}(I-\M \mathcal{H})^t \M \big(b\otimes_b b
    \notag \\ &\quad 
    + \lim_{i\to\infty}\mathbb{E} \bm{s}_{(i+1)T-t} \otimes_b \bm{s}_{(i+1)T-t} \big) 
    \notag \\ &\quad 
    -\hspace{-1mm} 2 \mathbb{E} \bm{\mathcal{A}}_{(i+1)T}^\tran (I\hspace{-1mm}-\hspace{-1mm}\M \mathcal{H})^T \otimes_b 
    \bm{\mathcal{A}}_{(i+1)T}^\tran \sum_{t=0}^{T-1}(I\hspace{-1mm}-\hspace{-1mm}\M \mathcal{H})^t
    \notag \\ &\quad \times
     \M \lim_{i\to\infty}\mathbb{E}\wse'_{iT} \otimes_b b,  
\end{align}
but:
\begin{align}
    &\lim_{i\to \infty} \mathbb{E} \bm{s}_{(i+1)T-t} \otimes_b \bm{s}_{(i+1)T-t} 
    \notag \\
    &=  \lim_{i\to \infty} \mathbb{E}  \text{bvec} \left(\diag{R_{k,(i+1)T-t} (\ws_{(i+1)T-t-1})}\right),
\end{align}
and using Jensen's inequality and Assumption \ref{assum:noiseProc}: 
\begin{align}
    &\Vert \diag{R_{k,(i+1)T-t}(w^o) - \mathbb{E}R_{k,(i+1)T-t}(\w_{k,(i+1)T-t-1}) } \Vert 
    \notag \\
    &\leq \mathbb{E}\Vert \diag{R_{k,(i+1)T-t}(w^o) \hspace{-1mm}-\hspace{-1mm} R_{k,(i+1)T-t}(\w_{k,(i+1)T-t-1}) } \Vert 
    \notag \\
    &\leq \kappa_s \mathbb{E}\Vert \wse_{(i+1)T-t-1}\Vert^{\alpha_s}
    \notag \\
    & = \kappa_s \mathbb{E} (\Vert \wse_{(i+1)T-t-1}\Vert^4)^{\gamma_s/4}
    \notag \\
    &\leq \kappa_s (\mathbb{E} \Vert \wse_{(i+1)T-t-1}\Vert^4)^{\gamma_s/4}.
\end{align}
Then, from Theorem \ref{thrm:4thStab} we conclude:
\begin{align}
    &\limsup_{i\to \infty} \Vert \diag{R_{k,(i+1)T-t}(w^o) 
    \notag \\ & \quad 
    - \mathbb{E}R_{k,(i+1)T-t}(\w_{k,(i+1)T-t-1}) } \Vert 
    \leq O(\mu^{\alpha_s/2}).
\end{align}
Therefore:
\begin{align}
    \lim_{i\to \infty} \mathbb{E} \bm{s}_{(i+1)T-t} \otimes_b \bm{s}_{(i+1)T-t} \hspace{-1mm}&=\hspace{-1mm} \text{bvec}(\diag{R_k}) \hspace{-1mm}+\hspace{-1mm} O(\mu^{\alpha_s/2}).
\end{align}
Finally, putting all the results together we get:
\begin{align}
    y_{\infty}
    &=
    \mathbb{E} \bm{\mathcal{A}}_{(i+1)T}^\tran \sum_{t=0}^{T-1}(I-\M \mathcal{H})^t\M \otimes_b \bm{\mathcal{A}}_{(i+1)T}^\tran \
    \notag \\ &\quad \times
    \sum_{t=0}^{T-1}(I-\M \mathcal{H})^t \M \big(b\otimes_b b  + \text{bvec}(\diag{R_k}) 
    \notag \\ &\quad 
   +O(\mu^{\alpha_s/2}) \big) +2\mu \mathbb{E} \bm{\mathcal{A}}_{(i+1)T}^\tran (I-\M \mathcal{H})^T
    \notag \\ &\quad 
     \otimes_b 
    \bm{\mathcal{A}}_{(i+1)T}^\tran \sum_{t=0}^{T-1}(I-\M \mathcal{H})^t
    \M 
    \notag \\ &\quad \times 
    (I-\bar{\mathcal{B}}_T)^{-1} \sum_{t=0}^{T-1} \bar{\mathcal{B}}_t b \otimes_b b,  
\end{align}
and since $\Vert I-\mathcal{G}\Vert = O(\mu)$:
\begin{align}
    z_{\infty} &= (I-\mathcal{G})^{-1} y_{\infty}
    = z + O(\mu^{1+\alpha_s/2}),
\end{align}
with:
\begin{align}\label{eq:defZ}
    z &\eqdef (I\hspace{-0.8mm}-\hspace{-0.8mm}\mathcal{G})^{-1} \Bigg(\mathbb{E} \bm{\mathcal{A}}_{(i+1)T}^\tran \sum_{t=0}^{T-1}(I\hspace{-0.8mm}-\hspace{-0.8mm}\M \mathcal{H})^t\M \otimes_b \bm{\mathcal{A}}_{(i+1)T}^\tran \
    \notag \\ &\quad \times
    \sum_{t=0}^{T-1}(I-\M \mathcal{H})^t \M \big(b\otimes_b b
    + \text{bvec}(\diag{R_k})  \big) 
    \notag \\ &\quad 
    +2\mu \mathbb{E} \bm{\mathcal{A}}_{(i+1)T}^\tran (I\hspace{-0.8mm}-\hspace{-0.8mm}\M \mathcal{H})^T \hspace{-0.8mm}\otimes_b \hspace{-0.8mm}
    \bm{\mathcal{A}}_{(i+1)T}^\tran \sum_{t=0}^{T-1}(I\hspace{-0.8mm}-\hspace{-0.8mm}\M \mathcal{H})^t  
    \notag \\ &\quad  \times 
    \M (I-\bar{\mathcal{B}}_T)^{-1} \sum_{t=0}^{T-1} \bar{\mathcal{B}}_t b \otimes_b b \Bigg).
\end{align}

Then, since:
\begin{align}
    \mathbb{E} \Vert \wse'_{(i+1)T}\Vert^2_{\Sigma} &= \mathbb{E} \text{Tr}\{\wse'_{(i+1)T}\wse'^\tran_{(i+1)T} \Sigma\} = z_{i+1}^\tran \text{bvec}(\Sigma),
\end{align}
we can show:
\begin{align}
    \lim_{i\to\infty}\mathbb{E} \Vert \wse'_{(i+1)T}\Vert^2 &= z_{\infty}^\tran \text{bvec}(I).
\end{align}
Using Theorem \ref{thrm:approxErr}, we conclude:
\begin{align}
    \lim_{i\to \infty} \mathbb{E}\Vert \wse_{(i+1)T}\Vert^2 &= \lim_{i\to \infty} \mathbb{E}\Vert \wse'_{(i+1)T}\Vert^2 
    \notag \\ &\quad 
    + \mathbb{E} \Vert \wse'_{(i+1)T} - \wse_{(i+1)T}\Vert^2 
    \notag \\
    &\quad + 2\mathbb{E}\wse'^\tran _{(i+1)T} (\wse_{(i+1)T }-\wse'_{(i+1)T})
    \notag \\
    &= \lim_{i\to \infty} \mathbb{E}\Vert \wse'_{(i+1)T}\Vert^2  + O(\mu^{3/2}),
\end{align}
and for $\alpha = 1+\frac{1}{2}\min\{1,\alpha_s\}$:
\begin{align}
    \text{MSD} &= \lim_{i\to \infty}\frac{1}{K}\sum_{k=1}^K \mathbb{E} \Vert \we_{k,(i+1)T}\Vert^2 
    \notag \\
    &= \frac{1}{K}z^\tran \text{bvec}(I) + O(\mu^{\alpha}).
\end{align}

\bibliographystyle{IEEEtran}
\bibliography{references}

\end{document}